\documentclass{article} 
\usepackage{nips13submit_e,times}
\usepackage{hyperref}
\usepackage{url}
\usepackage{microtype}

\usepackage{amsmath}
\usepackage{amsthm}
\usepackage{amssymb}
\usepackage[sort&compress]{natbib}

\bibpunct[; ]{[}{]}{,}{n}{}{;}

\usepackage{keyval}
\usepackage[ruled,vlined]{algorithm2e}
\usepackage{bm}
\usepackage{mathrsfs}
\usepackage{xspace}
\usepackage{array}
\usepackage{graphicx}
\usepackage{enumerate}

\newcommand{\BEAS}{\begin{eqnarray*}}
\newcommand{\EEAS}{\end{eqnarray*}}
\newcommand{\BEA}{\begin{eqnarray}}
\newcommand{\EEA}{\end{eqnarray}}
\newcommand{\BEQ}{\begin{equation}}
\newcommand{\EEQ}{\end{equation}}
\newcommand{\BIT}{\begin{itemize}}
\newcommand{\EIT}{\end{itemize}}
\newcommand{\BNUM}{\begin{enumerate}}
\newcommand{\ENUM}{\end{enumerate}}
\newcommand{\BA}{\begin{array}}
\newcommand{\EA}{\end{array}}

\newcommand{\rb}{\mathbb{R}}
\newcommand{\lova}{Lov\'asz }

\newtheorem{proposition}{Proposition}

\newcommand{\mysec}[1]{Section~\ref{sec:#1}}
\newcommand{\eq}[1]{Eq.~(\ref{eq:#1})}

\newcommand{\union}{\cup}
\newcommand{\inter}{\cap}

\newtheorem{theorem}{Theorem}
\newtheorem{lemma}[theorem]{Lemma}

\theoremstyle{definition}

\newcommand{\pnorm}[2]{\| {#1} \|_{#2}}

\newcommand{\enorm}[1]{\pnorm{#1}{2}}

\newcommand{\nlsum}{\sum\nolimits}

\newcommand{\Hc}{\mathcal{H}}

\newcommand{\reals}{\mathbb{R}}
\newcommand{\Ac}{\mathcal{A}}
\newcommand{\Bc}{\mathcal{B}}

\newcommand{\half}{\tfrac{1}{2}}
\newcommand{\set}[1]{\lbrace #1 \rbrace}

\DeclareMathOperator{\relint}{ri}
\DeclareMathOperator*{\argmin}{argmin}
\DeclareMathOperator*{\argmax}{argmax}

\DeclareMathOperator{\dom}{dom}
\DeclareMathOperator{\id}{I}
\DeclareMathOperator{\prox}{prox}

\nipsfinaltrue
 
\title{Reflection methods for user-friendly \\ submodular optimization
 }

\author{
Stefanie Jegelka\\
UC Berkeley\\
Berkeley, CA, USA\\
\And
Francis Bach \\
INRIA - ENS\\
Paris, France\\
\And
Suvrit Sra\\
MPI for Intelligent Systems\\
T\"ubingen, Germany\\
}

\nipsfinalcopy 

\begin{document}

\maketitle

\begin{abstract}
  Recently, it has become evident that submodularity naturally captures widely occurring concepts in machine learning, signal processing and computer vision. Consequently, there is need for efficient optimization procedures for submodular functions, especially for minimization problems.
  While general submodular minimization is challenging, we propose a new method that exploits existing decomposability of submodular functions. In contrast to previous approaches, our method is neither approximate, nor impractical, nor does it need any cumbersome parameter tuning. Moreover, it is easy to implement and parallelize.
  A key component of our method is a formulation of the discrete submodular minimization problem as a continuous best approximation problem that is solved through a sequence of reflections, and its solution can be easily thresholded to obtain an optimal discrete solution. 
  This method solves \emph{both} the continuous and discrete formulations of the problem, and therefore has applications in learning, inference, and reconstruction.
 In our experiments, we illustrate the benefits of our method on two image segmentation tasks.
 
 \end{abstract}

\section{Introduction}
Submodularity is a rich combinatorial concept that expresses widely occurring phenomena such as diminishing marginal costs and preferences for grouping. A set function $F: 2^{V} \to \mathbb{R}$ on a set $V$ is \emph{submodular} if for all subsets $S, T \subseteq V$, we have $F(S \union T) + F(S \inter T) \leq F(S) + F(T)$. 

Submodular functions underlie the goals of numerous problems in machine learning, computer vision and signal processing~\cite{bach2011learning}. Several problems in these areas can be phrased as submodular optimization tasks: notable examples include graph cut-based image segmentation~\cite{boykov2001fast}, sensor placement~\cite{krause11submodularity}, or document summarization~\cite{lin2011-class-submod-sum}. A longer list of examples may be found in~\cite{bach2011learning}. 

The theoretical complexity of submodular optimization is  well-understood: unconstrained minimization of submodular set functions is polynomial-time~\cite{fujishige2005submodular} while submodular maximization is NP-hard.
Algorithmically, however, the picture is different. Generic submodular maximization admits efficient algorithms that can attain \emph{approximate} optima with global guarantees; these algorithms are typically based on local search techniques~\cite{nemhauser1978analysis,feige2007maximizing}. 
In contrast, although polynomial-time solvable, submodular function minimization (SFM) which seeks to solve
\BEQ
\label{eq:sfm}
\min_{S \subseteq V} F(S),
\EEQ
poses substantial algorithmic difficulties. This is partly due to the fact that one is commonly interested in an exact solution (or an arbitrarily close approximation thereof), and ``polynomial-time'' is not necessarily equivalent to ``practically fast''.

Submodular minimization algorithms may be obtained from two main perspectives: \emph{combinatorial} and \emph{continuous}. Combinatorial algorithms for SFM typically use close connections to matroid and maximum flow methods; the currently theoretically fastest  combinatorial algorithm for SFM scales as $O(n^6 + n^5\tau)$, where $\tau$ is the time to evaluate the function oracle \cite{orlin2009faster} (for an overview of other algorithms, see e.g.,~\cite{mccormick2005submodular}). These combinatorial algorithms are typically nontrivial to implement.

Continuous methods offer an alternative by instead minimizing a \emph{convex extension}. This idea exploits the fundamental connection between a submodular function $F$ and its  \emph{Lov\'asz extension} $f$~\cite{lovasz1982submodular}, which is continuous and convex. The SFM problem~\eqref{eq:sfm} is then equivalent to 
\BEQ
\label{eq:min}
\min_{x \in [0,1]^n}f(x).
\EEQ
The Lov\'asz extension $f$ is nonsmooth, so we might have to resort to subgradient methods. While a fundamental result of Edmonds \citep{edmonds} demonstrates that a subgradient of $f$ can be computed in $O(n\log n)$ time, subgradient methods can be sensitive to choices of the step size, and can be slow. They theoretically converge at a rate of $O(1/\sqrt{t})$ (after $t$ iterations). 
The ``smoothing technique'' of \citep{nesterov2005smooth} does not in general apply here 
because computing a smoothed gradient is equivalent to solving the submodular minimization problem. We discuss this issue further in \mysec{review}.

An alternative to minimizing the \lova extension directly on $[0,1]^n$
is to consider a slightly modified convex problem. Specifically, the exact solution of the discrete problem $\min_{S \subseteq V} F(S)$ and of its nonsmooth convex relaxation $\min_{x \in [0,1]^n}f(x)$ may be found as a level set $S_0 = \{ k \mid x^\ast_k \geqslant 0\}$ of the unique point $x^*$ that minimizes the strongly convex function~\cite{chambolle2009total,bach2011learning}:
\begin{equation}
  \label{eq:smoothmin}
  f(x) + \half\|x\|^2.
\end{equation}
 We will refer to the minimization of~\eqref{eq:smoothmin} as the \emph{proximal} problem due to its close similarity to proximity operators used in convex optimization~\cite{comPes11}. When $F$ is a cut function, \eqref{eq:smoothmin} becomes a total variation problem (see, e.g.,~\cite{chambolle2004algorithm} and references therein) that also occurs in other regularization problems~\cite{bach2011learning}. Two noteworthy points about~\eqref{eq:smoothmin} are: (i) addition of the  strongly convex component $\half\|x\|^2$; (ii) the ensuing removal of the \emph{box-constraints} $x\in[0,1]^n$. These changes allow us to consider a convex dual which is amenable to smooth optimization techniques. 

Typical approaches to generic SFM include Frank-Wolfe methods \cite{frankwolfe56} that have cheap iterations and $O(1/t)$ convergence, but can be quite slow in practice (\mysec{experiments}); or the minimum-norm-point/Fujishige-Wolfe algorithm \cite{fujishigeIsotani11} that has expensive iterations but finite convergence. Other recent methods are approximate \cite{jegelkaLB11}. In contrast to several iterative methods based on convex relaxations, 
we seek to obtain exact discrete solutions.  

To the best of our knowledge, all generic algorithms that use only submodularity are several orders of magnitude slower than specialized algorithms when they exist (e.g., for graph cuts). However, the submodular function is not always generic and given via a black-box, but has known structure. Following~\cite{komodakis2011mrf,kolmogorov12,stobbe11,savchynskyy11}, we make the assumption that $F(S) = \sum_{i=1}^r F_i(S)$ is a sum of sufficiently ``simple'' functions (see Sec.~\ref{sec:dec}). 
This structure allows the use of (parallelizable) dual decomposition techniques for the problem in \eq{min}, with~\cite{savchynskyy11,chudak07} or without~\cite{komodakis2011mrf} Nesterov's smoothing technique, or with direct smoothing~\cite{stobbe11} techniques. 
 But existing approaches typically have two drawbacks: (a) they use smoothing or step-size parameters whose selection may be critical and quite tedious; and (b) they still exhibit slow convergence (see \mysec{simulations}).
 
These drawbacks arise from working with formulation \eqref{eq:min}. Our main insight is that, despite seemingly counter-intuitive, the proximal problem \eqref{eq:smoothmin} offers a much more user-friendly tool for solving \eqref{eq:sfm} than its natural convex counterpart~\eqref{eq:min}, both in implementation and running time. We approach Problem \eqref{eq:smoothmin} via its dual. This allows decomposition techniques which combine well with orthogonal projection and reflection methods that (a) exhibit faster convergence, (b) are easily parallelizable, (c) require no extra hyperparameters, and (d) are extremely easy to implement. 

The main three algorithms that we consider are: (i) dual block-coordinate descent (equivalently, primal-dual proximal-Dykstra), which was already shown to be extremely efficient for total variation problems~\cite{barbero2011fast} that are special cases of Problem~(\ref{eq:smoothmin}); (ii) Douglas-Rachford splitting using the careful variant of~\cite{bauschke2004finding}, which  for our formulation (Section~\ref{sec:DR}) requires no hyper-parameters; and (iii) accelerated projected gradient~\cite{fista}.
We will see these alternative algorithms can offer speedups beyond known efficiencies. Our observations have two implications: first, from the viewpoint of solving Problem \eqref{eq:smoothmin}, they offers speedups for often occurring denoising and reconstruction problems that employ total variation. Second, our experiments suggest that projection and reflection methods can work very well for solving the combinatorial problem~\eqref{eq:sfm}.  

In summary, we make the following contributions:
 
(1) In \mysec{dec}, we cast the problem of minimizing decomposable submodular functions as an orthogonal projection problem and show how existing optimization techniques may be brought to bear on this problem, to obtain fast, easy-to-code and easily parallelizable algorithms. In addition, we show examples of classes of functions amenable to our approach. In particular, for \emph{simple} functions, i.e., those for which minimizing $F(S) - a(S)$ is easy for all vectors\footnote{Every vector $a \in \mathbb{R}^n$ may be viewed as a modular (linear) set function: $a(S) \triangleq \sum_{i \in S}a(i)$. } $a \in \rb^n$, the problem in \eq{smoothmin} may be solved in $O(\log \frac{1}{\varepsilon})$ calls to such minimization routines, to reach a precision $\varepsilon$ (\mysec{review}, \ref{sec:dec}, Appendix~\ref{sec:divconquer}). 
%
(2) In \mysec{experiments},  we  demonstrate the empirical gains of using accelerated proximal methods, Douglas-Rachford and block coordinate descent methods over existing approaches: fewer hyperparameters and faster convergence.

\vspace*{-4pt}
\section{Review of relevant results from submodular analysis}
\label{sec:review}
\vspace*{-4pt}
The relevant concepts we review here are the \lova extension, base polytopes of submodular functions, and relationships between proximal and discrete problems.
For more details, see~\cite{fujishige2005submodular,bach2011learning}.

\paragraph{\lova extension and convexity.} 
The power set $2^V$ may be naturally identified with the vertices of the hypercube, i.e., $\{0,1\}^n$. The \lova extension $f$ of any set function is defined by linear interpolation, so that for any $S \subset V$, $F(S) = f(1_S)$. It may be computed in closed form once the components of $x$ are sorted:
if  $x_{\sigma(1)} \geqslant \cdots \geqslant x_{\sigma(n)}$, then $f(x) = \sum_{k=1}^n x_{\sigma(k)} \big[ F( \{ \sigma(1),\dots, \sigma(k) \} ) -  F( \{ \sigma(1),\dots, \sigma(k-1) \} ) \big]$ \cite{lovasz1982submodular}. 
For the graph cut function, $f$ is the total variation.

In this paper, we are going to use two important results: (a) if the set function $F$ is submodular, then its \lova extension $f$ is convex, and (b) minimizing the set function $F$ is equivalent to minimizing $f(x)$ with respect to $x \in [0,1]^n$. Given $x \in [0,1]^n$, all of its level sets may be considered and the function may be evaluated (at most $n$ times) to obtain a set $S$. 
Moreover, for a submodular function, the \lova extension happens to be the support function of the base polytope $B(F)$ defined as
$$
B(F) = \{ y \in \rb^n \mid  \ \forall S \subset V, \  y(S) \leqslant F(S)   \mbox{ and }  \ y(V) = F(V)  \},
$$
that is $f(x) = \max_{ y\in B(F)} y^\top x$ \cite{edmonds}. A maximizer of $y^\top x$ (and hence the value of $f(x)$), may be computed by the ``greedy algorithm'', which first sorts the components of $w$ in decreasing order $x_{\sigma(1)} \geqslant \cdots \geqslant x_{\sigma(n)}$, and then compute $y_{\sigma(k)} = F( \{ \sigma(1),\dots, \sigma(k) \} ) -  F( \{ \sigma(1),\dots, \sigma(k-1) \} ) $. In other words, a linear function can be maximized over $B(F)$ in time $O(n\log n + n \tau)$ (note that the term $n\tau$ may be improved in many special cases). This is crucial for exploiting convex duality.

\vspace*{-4pt}
\paragraph{Dual of discrete problem.}
We may derive a dual problem to the discrete problem in \eq{sfm} and the convex nonsmooth problem in \eq{min}, as follows:
\begin{equation}
\label{eq:dual}
\min_{S \subseteq V } F(S) = 
\min_{x \in [0,1]^n} f(x)  
=  \min_{x \in [0,1]^n} \max_{y \in B(F)} y^\top x 
=  \max_{y\in B(F)}   \min_{x \in [0,1]^n} y^\top x  
=  \max_{y \in B(F)}  (y)_-(V),
\end{equation}
where $(y)_{-} = \min\{y,0\}$ applied elementwise.
This allows to obtain dual certificates of optimality from any $y\in B(F)$ and $x \in [0,1]^n$.

\vspace*{-4pt}
\paragraph{Proximal problem.} The optimization problem \eqref{eq:smoothmin}, i.e., $\min_{x \in \rb^n}  f(x) + \frac{1}{2}\|x\|^2$, has intricate relations to the SFM problem~\cite{chambolle2009total}. Given the unique optimal solution $x^*$ of \eqref{eq:smoothmin}, the maximal (resp. minimal) optimizer of the SFM problem is the set $S^*$ of nonnegative (resp. positive) elements of~$x^*$.
More precisely, solving~\eqref{eq:smoothmin} is equivalent to minimizing $F(S)  + \mu |S|$ for all $\mu \in \rb$. A solution $S_\mu \subseteq V$ is obtained from a solution $x^\ast$ as $S_\mu^\ast = \{i\mid x^\ast_i \geqslant \mu\}$. Conversely, $x^\ast$ may be obtained from all $S_\mu^\ast$ as $x_k^\ast = \sup \{ \mu \in \rb \mid \ k \in S_\mu^\ast \}$ for all $k \in V$. Moreover, if $x$ is an $\varepsilon$-optimal solution of \eq{smoothmin}, then we may construct $\sqrt{\varepsilon n}$-optimal solutions for all $S_\mu$~\cite[Prop.~10.5]{bach2011learning}. In practice, the duality gap of the discrete problem is usually much lower than that of the proximal version of the same problem, as we will see in \mysec{experiments}. Note that the problem in \eq{smoothmin} provides much more information than \eq{min}, as all $\mu$-parameterized discrete problems are solved. 

The dual problem of Problem \eqref{eq:smoothmin} reads as follows:
$$
\min_{x \in \rb^n} f(x)   + \tfrac{1}{2} \| x \|_2^2
=  \min_{x \in \rb^n} \max_{y \in B(F)} y^\top x   + \tfrac{1}{2} \| x\|_2^2 
=  \max_{y \in B(F)}   \min_{x \in \rb^n} y^\top x  + \tfrac{1}{2} \| x\|_2^2 
=  \max_{y \in B(F)}  -\tfrac{1}{2} \| y  \|_2^2 ,
$$
where primal and dual variables are linked as $x = -y$. Observe that this dual problem is equivalent to finding the orthogonal projection of $0$ onto $B(F)$.

\paragraph{Divide-and-conquer strategies for the proximal problems.}
Given a solution $x^\ast$ of the proximal problem, we have seen how to get $S_\mu^\ast$ for any $\mu$ by simply thresholding $x^\ast$ at $\mu$. Conversely, one can recover $x^\ast$ exactly from at most $n$ well-chosen values of $\mu$. A known divide-and-conquer strategy \cite{groenevelt1991two,fujishige2005submodular} hinges upon the fact that for any $\mu$, one can easily see which components of $x^\ast$ are greater or smaller than $\mu$ by computing $S^\ast_\mu$. The resulting algorithm makes $O(n)$ calls to the submodular function oracle. In Appendix \ref{sec:divconquer}, we extend an alternative approach by \citet{tarjan2006balancing} for cuts to general submodular functions and obtain a solution to \eqref{eq:smoothmin} up to precision $\varepsilon$ in $O(\min\{n, \log \frac{1}{\varepsilon}\})$ iterations. This result is particularly useful if our function $F$ is a sum of functions for each of which by itself the SFM problem is easy. Beyond squared $\ell_2$-norms, our algorithm equally applies to computing all minimizers of $f(x) + \sum_{j=1}^p h_j(x_j)$ for arbitrary smooth strictly convex functions $h_j$, $j=1,\dots,n$.

\vspace*{-4pt}
\section{Decomposition of submodular functions}
\label{sec:dec}
\vspace*{-4pt}
Following~\cite{komodakis2011mrf,kolmogorov12,stobbe11,savchynskyy11},
we assume that our  function $F$ may be decomposed as the sum $F(S) = \sum_{j=1}^r F_j(S)$ of $r$ ``simple'' functions.
In this paper, by ``simple'' we mean functions $G$ for which $G(S) - a(S)$ can be minimized efficiently for all vectors $a \in \rb^n$ (more precisely, we require that $S \mapsto G(S \cup T) - a(S)$ can be minimized efficiently over all subsets of $V\setminus T$, for any $T \subseteq V$ and~$a$). Efficiency may arise from the functional form of $G$, or from the fact that $G$ has small support.
For such functions, Problems \eqref{eq:sfm} and \eqref{eq:smoothmin} become
\begin{equation}
 \min_{S \subseteq V} \nlsum_{j=1}^r F_j(S) = \min_{x \in [0,1]^n}  \nlsum_{j=1}^r f_j(x) \quad\qquad
  \min_{x \in \rb^n} \nlsum_{j=1}^r f_j(x) + \half \| x\|_2^2.
\label{eq:2}
\end{equation}
The key to the algorithms presented here is to be able to minimize $\half \| x  - z \|_2^2 + f_j(x)$, or equivalently, to orthogonally project $z$ onto $B(F_j)$: 
 $\min \half \| y - z\|_2^2$ subject to $y \in B(F_j)$.


We next sketch some examples of functions $F$ and their decompositions into simple functions~$F_j$. As shown at the end of \mysec{review}, projecting onto $B(F_j)$ is easy as soon as the corresponding submodular minimization problems are easy. Here we outline some cases for which specialized fast algorithms are known.
 
\textbf{Graph cuts.}
A widely used class of submodular functions are graph cuts. Graphs may be decomposed into substructures such as trees, simple paths or single edges. Message passing algorithms apply to trees, while the proximal problem for paths is very efficiently solved by \cite{barbero2011fast}. 
For single edges, it is solvable in closed form.
Tree decompositions are common in graphical models, whereas path decompositions are frequently used for TV problems \cite{barbero2011fast}.\\[4pt]
\textbf{Concave functions.} Another important class of submodular functions is that of concave functions of cardinality, i.e., $F_j(S) = h(|S|)$ for a concave function $h$. Problem~\eqref{eq:smoothmin} for such functions may be solved in $O(n\log n)$ time (see \cite{fujishige80} and Appendix~\ref{sec:divconquer}). Functions of this class have been used in \cite{stobbe11,jegelkaLB11,kohli09}. Such functions also include covering functions \cite{stobbe11}.\\[4pt]
\textbf{Hierarchical functions.} Here, the ground set corresponds to the leaves of a rooted, undirected tree. Each node has a weight, and the cost of a set of nodes $S \subseteq V$ is the sum of the weights of all nodes in the smallest subtree (including the root) that spans $S$. This class of functions too admits to solve the proximal problem in $O(n\log n)$ time \cite{iwata04,hochbaum95}. Related tree functions have been considered in \cite{jenatton11}, where the elements $v$ of the ground set are arranged in a tree of height $d$ and each have a weight $w(v)$. Let $\mathrm{desc}(v)$ be the set of descendants of $v$ in the tree. Then $F(S) = \sum_{v \in V}w(v) 1[\mathrm{dec}(v) \inter S \neq \emptyset]$. \citet{jenatton11} show how to solve the proximal problem for such a function in time $O(nd)$.
\\[4pt]
\textbf{Small support.} Any general, potentially slower algorithm such as the minimum-norm-point algorithm can be applied if the support of each $F_j$ is only a small subset of the ground set.

\vspace*{-4pt}
\subsection{Dual decomposition of the nonsmooth problem}
\label{sec:ns.dualdecomp}
We first review existing dual decomposition techniques for the nonsmooth problem \eqref{eq:sfm}. We always assume that $F = \sum_{j=1}^rF_j$, and define $\Hc^r := \prod_{j=1}^r \reals^n \simeq \reals^{n\times r}$. We follow~\cite{komodakis2011mrf} to derive a dual formulation (see Appendix~\ref{sec:duals}):
\begin{lemma}\label{lem.dualnonsmooth}
  The dual of Problem~\eqref{eq:sfm} may be written in terms of variables $\lambda_1,\ldots,\lambda_r \in \reals^n$ as
  \begin{equation}\label{eq:nsdual}
    \max \nlsum_{j=1}^r g_j(\lambda_j) \qquad\text{s.t. }\;\; \lambda \in \bigl\{(\lambda_1,\dots,\lambda_r) \in \Hc^r \mid \nlsum_{j=1}^r \lambda_j = 0\bigr\}
  \end{equation}
  where $g_j(\lambda_j) = \min_{S \subset V} F_j(S) - \lambda_j(S)$ is  a  nonsmooth concave function.
\end{lemma}
The dual is the maximization of a nonsmooth concave  function over a convex set, 
onto which it is easy to project: the projection of a vector $y$ has $j$-th block equal to $y_j - \frac{1}{r} \sum_{k=1}^r y_k$. Moreover, in our setup, functions $g_j$ and their subgradients may be computed efficiently through SFM.

We consider several existing alternatives for the minimization of $f(x)$ on $x \in [0,1]^n$, most of which use Lemma~\ref{lem.dualnonsmooth}. Computing subgradients for any $f_j$ means calling the greedy algorithm, which runs in time $O(n\log n)$. All of the following algorithms require the tuning of an appropriate step size.

\textbf{Primal subgradient descent (primal-sgd)}: Agnostic to any decomposition properties, we may apply a standard simple subgradient method to $f$. A subgradient of $f$ may be obtained from the subgradients of the components $f_j$. This algorithm 
converges at rate $O(1/\sqrt{t})$.\\[4pt]
\textbf{Dual subgradient descent (dual-sgd)} \cite{komodakis2011mrf}: Applying a subgradient method to the nonsmooth dual in Lemma~\ref{lem.dualnonsmooth} leads to a convergence rate of $O(1/\sqrt{t})$. Computing a subgradient requires minimizing the submodular functions $F_j$ individually. In simulations, following \cite{komodakis2011mrf}, we consider a step-size rule similar to Polyak's rule (dual-sgd-P)~\cite{bertsekas1999nonlinear}, as well as a decaying step-size (dual-sgd-F), and use discrete optimization for all $F_j$.\\[4pt]
\textbf{Primal smoothing (primal-smooth)} \cite{stobbe11}: The nonsmooth primal may be smoothed in several ways by smoothing the $f_j$ individually; one example is
$\tilde{f}^\varepsilon_j(x_j) = \max_{y_j \in B(F_j)} y_j^\top x_j - \frac{\varepsilon}{2} \| y_j\|^2$. This leads to a function that is $(1/\varepsilon)$-smooth. Computing $\tilde{f}^\varepsilon_j$ means solving the proximal problem for~$F_j$. 
The convergence rate is $O(1/t)$, but, apart from the step size which may be set relatively easily, the smoothing constant $\varepsilon$ needs to be defined.\\[4pt]
\textbf{Dual smoothing (dual-smooth)}: Instead of the primal, the dual \eqref{eq:nsdual} may be smoothed, e.g., by entropy 
\cite{savchynskyy11,savchynskyy12} applied to each $g_j$ as
$\tilde{g}_j^\varepsilon(\lambda_j) =  \min_{x \in [0,1]^n} f_j(x) + \varepsilon h(x) $ where $h(x)$ is a negative entropy. 
Again, the convergence rate is $O(1/t)$ but there are two free parameters (in particular the smoothing constant $\varepsilon$ which is hard to tune). This method too requires solving proximal problems for all $F_j$ in each iteration.

Dual smoothing with entropy also admits coordinate descent methods \cite{meshi12} that exploit the decomposition, but we do not compare to those here.

\vspace*{-4pt}
\subsection{Dual decomposition methods for proximal problems}
\label{sec.dual.decomp}
\vspace*{-4pt}
We  may also consider \eq{smoothmin} and first derive a dual problem using the same technique as in \mysec{ns.dualdecomp}. Lemma~\ref{lem.dual} (proved in Appendix~\ref{sec:duals}) formally presents our dual formulation as a best approximation problem. The primal variable can be recovered as $x = -\sum_jy_j$.
\begin{lemma}
  \label{lem.dual}
  The dual of~\eq{smoothmin} may be written as the \emph{best approximation problem}
  \begin{align}
    \label{eq.1}
    \min_{\lambda,y}\quad\enorm{y-\lambda}^2\qquad\text{s.t. }\,& \lambda \in \bigl\{(\lambda_1,\dots,\lambda_r) \in \Hc^r \mid \nlsum_{j=1}^r \lambda_j = 0\bigr\},
     \qquad y \in \prod\nolimits_{j=1}^r B(F_j).\vspace*{-4pt}
  \end{align}
\end{lemma}
We can actually eliminate the $\lambda_j$ and obtain the simpler looking dual problem
\begin{equation}
  \label{eq:dualsmooth}
  \max_{y} - \frac{1}{2} \Bigl \| \nlsum_{j=1}^r y_j  \Bigr\|_2^2 \qquad \text{s.t. }\;\; y_j \in B(F_j), \ j \in \{1,\dots,r\}
\end{equation}
Such a dual was also used in \cite{stobbe13}.
In \mysec{experiments}, we will see the effect of solving one of these duals or the other.
For the simpler dual~\eqref{eq:dualsmooth} the case $r=2$ is of special interest; it reads
\begin{equation}
  \label{eq:dualsmoothr2}
  \max_{y_1 \in B(F_1), \ y_2 \in B(F_2)} - \frac{1}{2} \| y_1  + y_2 \|_2^2\quad\Longleftrightarrow \min_{y_1 \in B(F_1), -y_2 \in -B(F_2)}\enorm{y_1 - (-y_2)}.\vspace*{1pt}
\end{equation}
We write Problem~\eqref{eq:dualsmoothr2} in this suggestive form to highlight its key geometric structure: it is, like \eqref{eq.1}, a \emph{best approximation problem}: i.e., the problem of finding the closest point between the polytopes $B(F_1)$ and $-B(F_2)$. Notice, however, that \eqref{eq.1} is very different from~\eqref{eq:dualsmoothr2}---the former operates in a product space while the latter does not, a difference that can have impact in practice (see \mysec{experiments}).
We are now ready to present algorithms that exploit our dual formulations.

\vspace*{-6pt}
\section{Algorithms}
\label{sec.algo}
\vspace*{-4pt}
We describe a few competing methods for solving our smooth dual formulations. We describe the details for the special 2-block case~\eqref{eq:dualsmoothr2}; the same arguments apply to the block dual from Lemma~\ref{lem.dual}.

\subsection{Block coordinate descent or proximal-Dykstra}
\vspace*{-4pt}
Perhaps the simplest approach to solving~\eqref{eq:dualsmoothr2} (viewed as a minimization problem) is to use a block coordinate descent (BCD) procedure, which in this case performs the alternating projections:
\begin{equation}
  \label{eq.4}
    y_1^{k+1} \gets \argmin\nolimits_{y_1 \in B(F_1)} \enorm{y_1 - (-y_2^k)}^2;\quad\quad
    y_2^{k+1} \gets \argmin\nolimits_{y_2 \in B(F_2)} \enorm{y_2 -(-y_1^{k+1})}.
\end{equation}
The iterations for solving \eqref{eq:dualsmooth} are analogous.
This BCD method (applied to \eqref{eq:dualsmoothr2}) is equivalent to applying the so-called proximal-Dykstra method~\cite{comPes11} to the primal problem. This may be seen by comparing the iterates. 
Notice that the BCD iteration~\eqref{eq.4} is nothing but alternating projections onto the convex polyhedra $B(F_1)$ and $B(F_2)$. There exists a large body of literature studying method of alternating projections---we refer the interested reader to the monograph~\cite{deutsch} for further details. 

However, despite its attractive simplicity, it is known that BCD (in its alternating projections form), can converge arbitrarily slowly~\cite{bauschke2004finding} depending on the relative orientation of the convex sets onto which one projects. Thus, we turn to a potentially more effective method.

\subsection{Douglas-Rachford splitting}
\label{sec:DR}
The Douglas-Rachford (DR) splitting method~\cite{DR} includes algorithms like ADMM as a special case~\cite{comPes11}. It avoids the slowdowns alluded to above by replacing alternating projections with alternating ``reflections''. Formally, DR applies to convex problems of the form~\cite{comPes11,bauComb}
\begin{equation}
  \label{eq.6}
  \min\nolimits_x\quad \phi_1(x)+\phi_2(x),
\end{equation}
subject to the qualification $\relint(\dom \phi_1) \cap \relint(\dom \phi_2) \neq \varnothing$. To solve~\eqref{eq.6}, DR starts with some $z_0$, and performs the three-step iteration (for $k\ge 0$):
\begin{equation}
\label{eq.5}
  \begin{split}
    1.\ \ x_k = \prox_{\phi_2}(z_k);\qquad2.\ \ v_k = \prox_{\phi_1}(2x_k-z_k);\qquad3.\ \ z_{k+1} = z_k + \gamma_k(v_k-z_k),
  \end{split}
\end{equation}
where $\gamma_k \in [0,2]$ is a sequence of scalars that satisfy $\nlsum_k \gamma_k(2-\gamma_k) = \infty$. The sequence $\set{x_k}$ produced by iteration~\eqref{eq.5} can be shown to converge to a solution of~\eqref{eq.6} \cite[Thm.~25.6]{bauComb}. 

Introducing the \emph{reflection operator} $$R_{\phi} := 2\prox_\phi-\id,$$ and setting $\gamma_k=1$, the DR iteration~\eqref{eq.5} may be written in a more symmetric form as
\begin{equation}
  \label{eq.7}
  x_k = \prox_{\phi_2}(z_k),\qquad z_{k+1} = \half[R_{\phi_1}R_{\phi_2}+\id]z_k,\quad k \ge 0.
\end{equation}

Applying DR to the duals~\eqref{eq.1} or \eqref{eq:dualsmoothr2}, requires first putting them in the form~\eqref{eq.6}, either by introducing extra variables or by going back to the primal, which is unnecessary. This is where the special structure of our dual problem proves crucial, a recognition that is subtle yet remarkably important. 

Instead of applying DR to~\eqref{eq:dualsmoothr2}, consider the closely related problem
\begin{equation}
  \label{eq.3}
  \min\nolimits_{y}\quad \delta_1(y) + \delta_2^{-}(y),
\end{equation}
where $\delta_1$, $\delta_2^{-}$ are indicator functions for $B(F_1)$ and $-B(F_2)$, respectively. Applying DR directly to~\eqref{eq.3} does not work because usually
  $\relint(\dom \delta_1) \cap \relint(\dom \delta_2) = \varnothing$.
Indeed, applying DR to~\eqref{eq.3} generates iterates that diverge to infinity~\cite[Thm.~3.13(ii)]{bauschke2004finding}. Fortunately, even though the DR iterates for~\eqref{eq.3} may diverge, \citet{bauschke2004finding} show how to extract convergent sequences from these iterates, which actually solve the corresponding best approximation problem; for us this is nothing but the dual~\eqref{eq:dualsmoothr2} that we wanted to solve in the first place. Theorem~\ref{thm.aar}, which is a simplified version of~\cite[Thm.~3.13]{bauschke2004finding}, formalizes the above discussion.

\begin{theorem}{\protect{\cite{bauschke2004finding}}}
  \label{thm.aar}
  Let $\Ac$ and $\Bc$ be nonempty polyhedral convex sets. Let $\Pi_{\Ac}$ ($\Pi_{\Bc}$)  denote orthogonal projection onto $\Ac$ ($\Bc$), and let $R_{\Ac} := 2\Pi_{\Ac}-\id$ (similarly $R_{\Bc}$) be the corresponding reflection operator.   Let $\set{z_k}$ be the sequence generated by the DR method~\eqref{eq.7} applied to~\eqref{eq.3}. If $\Ac \cap \Bc \neq\varnothing$, then $\set{z_k}_{k \ge 0}$ converges weakly to a fixed-point of the operator $T := \half [R_{\Ac}R_{\Bc}+\id]$; otherwise $\enorm{z_k} \to \infty$. The sequences $\set{x_k}$ and $\set{\Pi_{\Ac}\Pi_{\Bc}z_k}$ are bounded; the weak cluster points of either of the two sequences
 \begin{equation}
   \label{eq.2}
   \set{(\Pi_{\Ac}R_{\Bc}z_k, x_k)}_{k\ge 0}\quad \set{(\Pi_{\Ac}x_k,x_k)}_{k \ge 0},
 \end{equation}
 are solutions best approximation problem $\min_{a,b}\|a-b\|$ such that $a\in \Ac$ and $b \in \Bc$.
\end{theorem}
 
The key consequence of Theorem~\ref{thm.aar} is that we can apply DR with impunity to~\eqref{eq.3}, and extract from its iterates the optimal solution to  problem~\eqref{eq:dualsmoothr2} (from which recovering the primal is trivial). The most important feature of solving the dual~\eqref{eq:dualsmoothr2} in this way is that absolutely no stepsize tuning is required,
making the method very practical and user friendly (see also Appendix~\ref{sec:recipe}).

\vspace*{-.3cm}
\section{Experiments}
\label{sec:results}
\label{sec:simulations}
\label{sec:experiments}
\vspace*{-.35cm}

We empirically compare the proposed projection methods\footnote{Code and data corresponding to this paper are available at \textit{https://sites.google.com/site/mloptstat/drsubmod}} to the (smoothed) subgradient methods discussed in Section~\ref{sec:ns.dualdecomp}. For solving the proximal problem, we apply block coordinate descent (BCD) and Douglas-Rachford (DR) to Problem \eqref{eq:dualsmooth} if applicable, and also to \eqref{eq.1} (BCD-para, DR-para). In addition, we use acceleration to solve~\eqref{eq:dualsmooth} or~\eqref{eq:dualsmoothr2}~\cite{fista}.
The main iteration cost of all methods except for the primal subgradient method is the orthogonal projection onto polytopes $B(F_j)$, and therefore the number of iterations is a suitable criterion for comparisons.
The primal subgradient method uses the greedy algorithm in each iteration, which runs in $O(n\log n)$. However, as we will see, its convergence is so slow to counteract any benefit that may arise from not using projections.
We do not include Frank-Wolfe methods here, since FW is equivalent to a subgradient descent on the primal and converges correspondingly slowly. 

As benchmark problems, we use (i) graph cut problems for segmentation, or MAP inference in a 4-neighborhood grid-structured MRF, and (ii) concave functions similar to those used in \cite{stobbe11}, but together with graph cut functions. The segmentation problems (i) are set up in a fairly standard way on a 4-neighbor grid graph, with unary potentials derived from Gaussian Mixture Models of color features. The weight of graph edge $(i,j)$ is a function of $\exp(-\|y_i - y_j\|^2)$, where $y_i$ is the RGB color vector of pixel $i$.
The functions in (i) decompose as sums over vertical and horizontal paths. All horizontal paths are independent and can be solved together in parallel, and similarly all vertical paths. The functions in (ii) are constructed by extracting regions $R_j$ via superpixels \cite{levinshtein09} and, for each $R_j$, defining the function $F_j(S) = |S||R_j\setminus S|$. We use 200 and 500 regions. The problems have size $640\times 427$. Hence, for (i) we have $r=640+427$ (but solve it as $r=2$) and for (ii) $r=640+427+500$ (solved as $r=3$).

For algorithms working with formulation \eqref{eq.1}, we compute an improved smooth duality gap of a current primary solution $x = -\sum_j y_j$ as follows: find $y' \in \argmax_{y \in B(F)} x^\top y$ (then $f(x) = x^\top y'$) and find an improved $x'$ by minimizing $\min_z z^\top y' + \half\|z\|^2$ subject to the constraint that $z$ has the same ordering as $x$ \cite{bach2011learning}. The constraint ensures that $(x')^\top y' = f(x')$. This is an isotonic regression problem and can be solved in time $O(n)$ using the ``pool adjacent violators'' algorithm \cite{bach2011learning}. The gap is then 
$f(x') + \half\|x'\|^2 - (-\half\|y'\|^2)$.

For computing the discrete gap, we find the best level set $S_i$ of $x$ and, using $y' = -x$, compute $\min_i F(S_i) - y'_{-}(V)$.

 \begin{figure}
   \centering
   {\footnotesize
   \begin{tabular}{c@{\hspace{0pt}}c@{\hspace{0pt}}cc@{\hspace{0pt}}c}
     & pBCD, iter 1 & pBCD, iter 7 & DR, iter 1 & DR, iter 4\\     
     \includegraphics[width=0.2\textwidth]{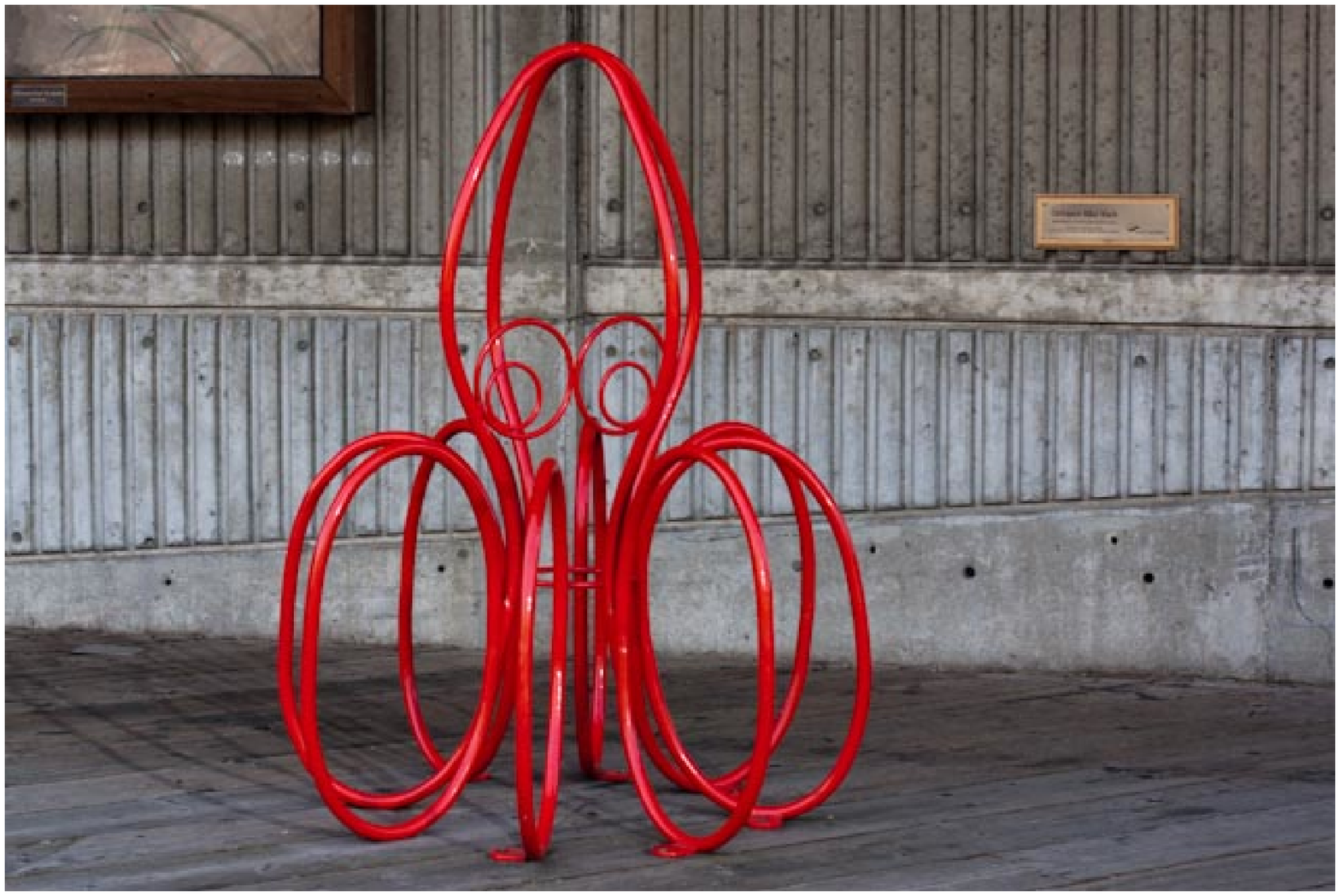} &
     \includegraphics[width=0.2\textwidth]{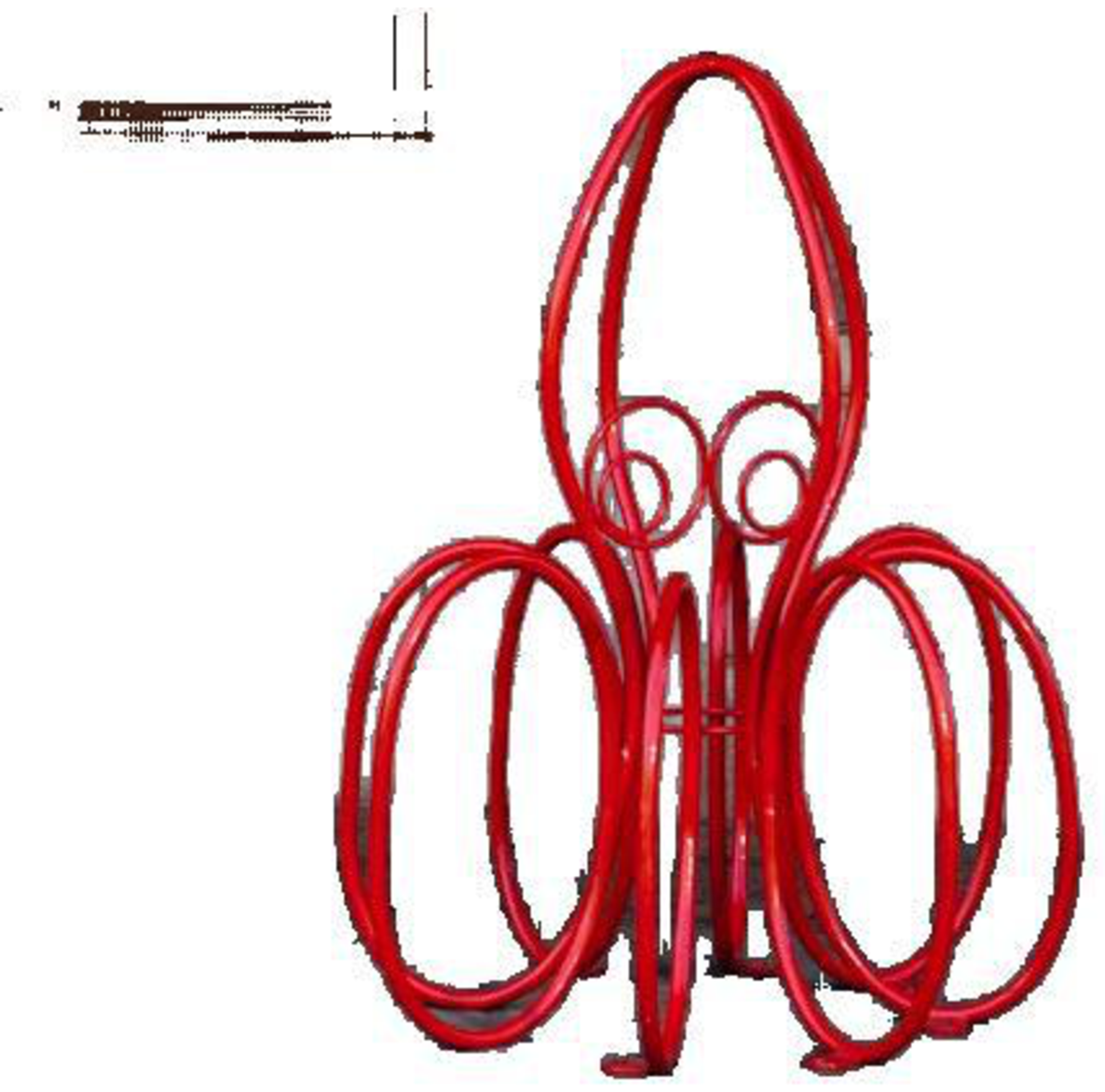} &
     \includegraphics[width=0.2\textwidth]{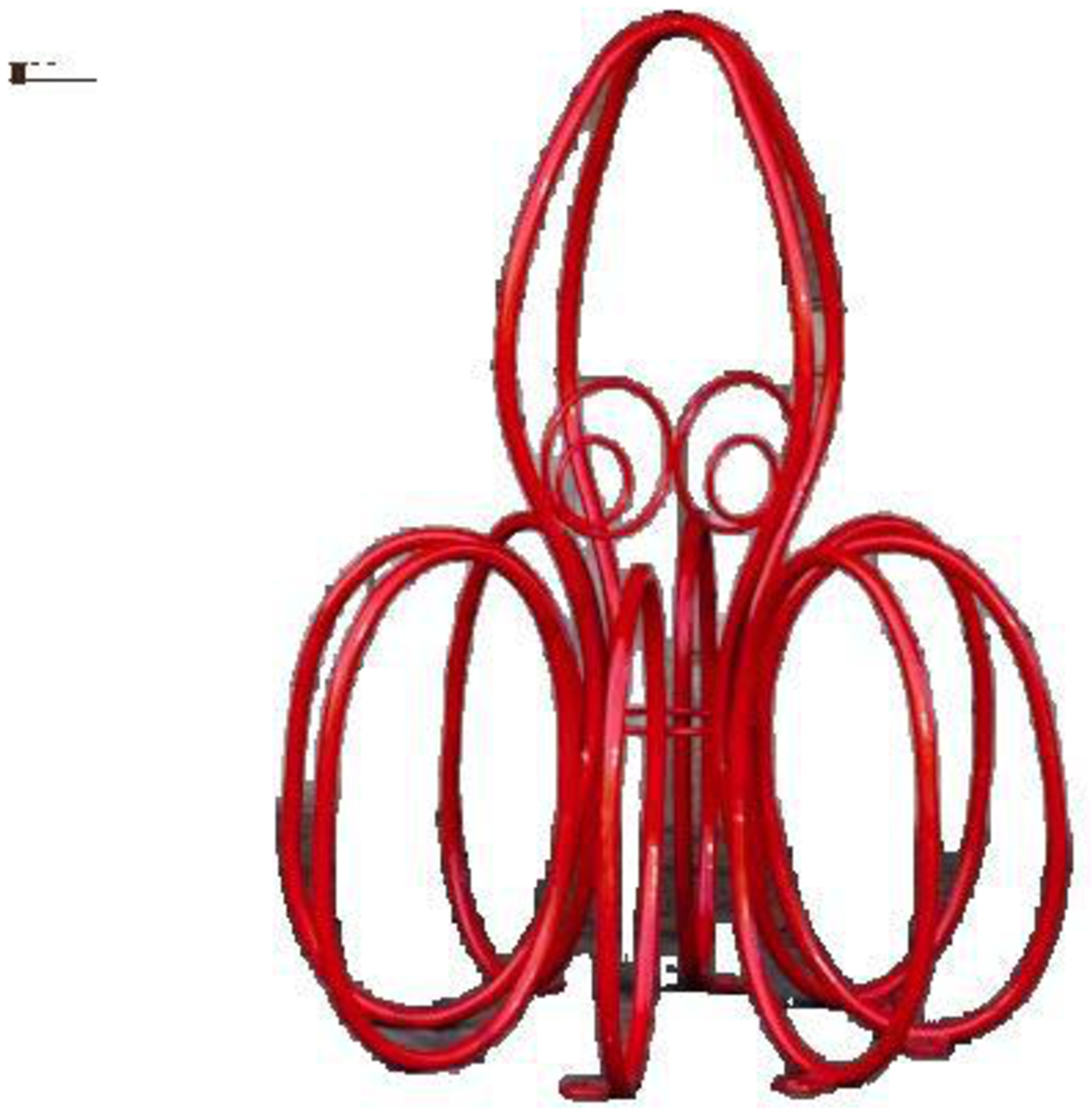}&
     \includegraphics[width=0.2\textwidth]{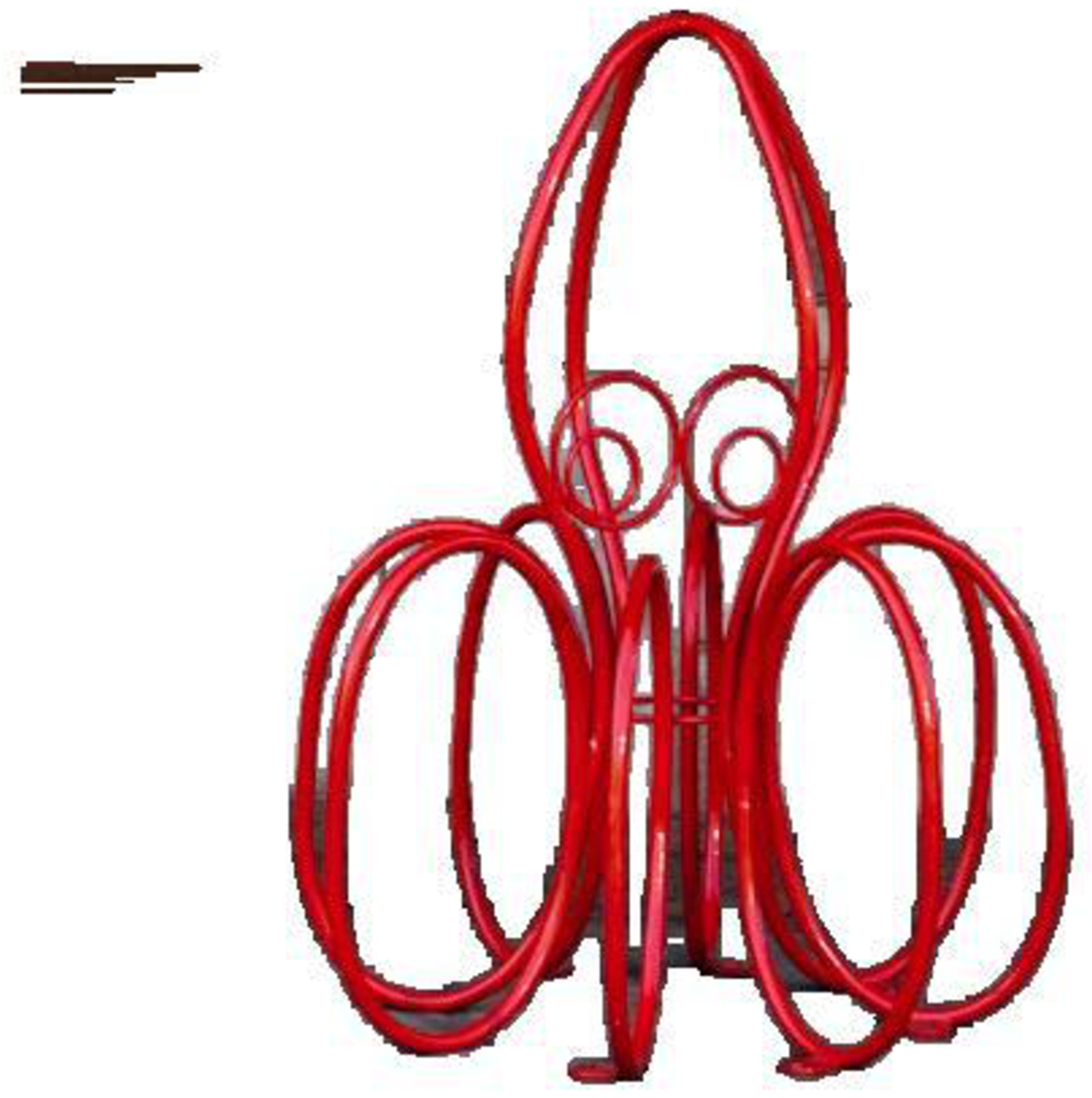}&
     \includegraphics[width=0.2\textwidth]{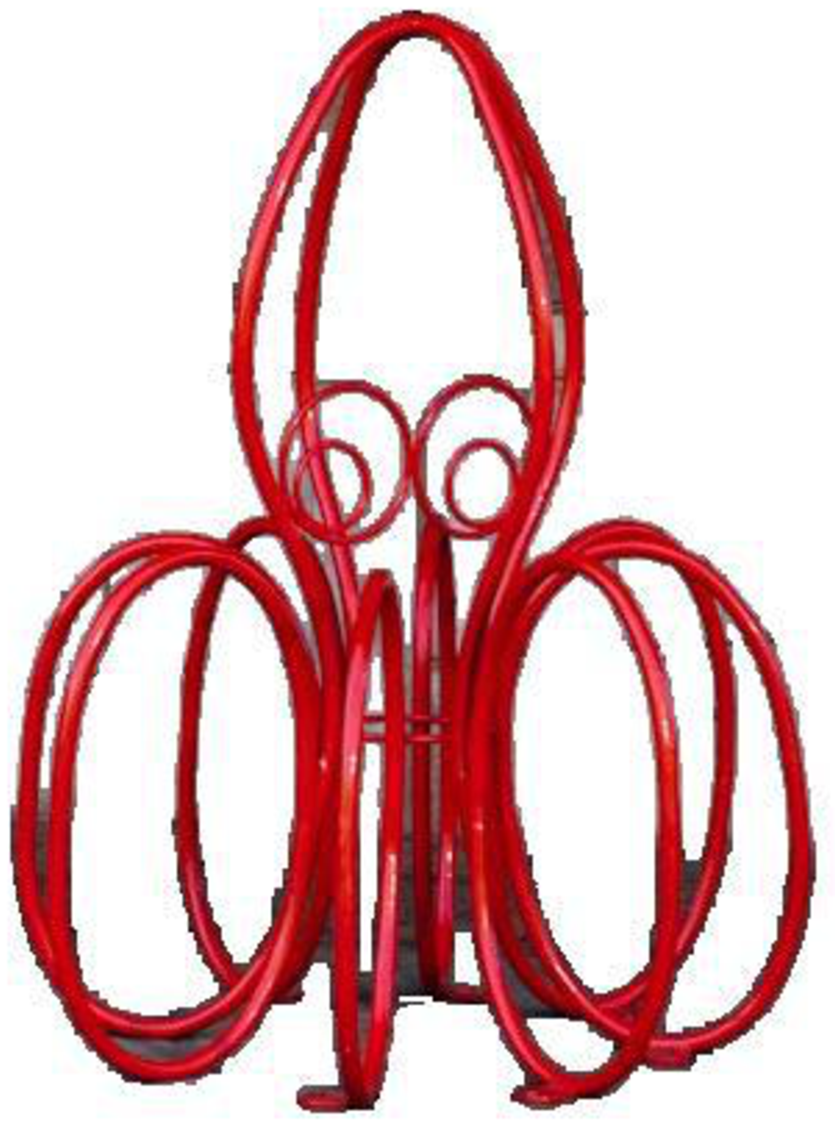}\\
     smooth gap &
     $\nu_s=3.4\cdot 10^{6}$ & 
     $\nu_s= 4.4\cdot 10^{5}$ & $\nu_s= 4.17\cdot 10^{5}$ & $\nu_s= 8.05 \cdot 10^{4}$ \\
     discrete gap & $\nu_d=4.6\cdot 10^3$ 
     & $\nu_d=5.5\cdot 10^{2}$ & $\nu_d=6.6\cdot 10^{3}$ & $\nu_d=5.9\cdot 10^{-1}$
   \end{tabular}
 }
 
   \vspace*{-.4cm}
   
   \caption{\small Segmentation results for the slowest and fastest projection method, with smooth ($\nu_s$) and discrete ($\nu_d$) duality gaps. Note how the background noise disappears only for small duality gaps.}
   \label{fig:segs}
     \vspace*{-.3cm}
 \end{figure}

\textbf{Two functions ($r=2$).} Figure~\ref{fig:numiters} shows the duality gaps for the discrete and smooth (where applicable) problems for two instances of segmentation problems. The algorithms working with the proximal problems are much faster than the ones directly solving the nonsmooth problem. In particular DR converges extremely fast, faster even than BCD which is known to be a state-of-the-art algorithms for this problem \cite{barbero2011fast}. This, in itself, is a new insight for solving TV. 
We also see that the discrete gap shrinks faster than the smooth gap, i.e., the optimal discrete solution does not require to solve the smooth problem to extremely high accuracy. Figure~\ref{fig:segs} illustrates example results for
different gaps.

\textbf{More functions ($r>2$).} Figure~\ref{fig:numiters2} shows example results for four problems of sums of concave and cut functions. Here, we can only run DR-para. Overall, BCD, DR-para and the accelerated gradient method perform very well.

\begin{figure}

   \vspace*{-.2cm}

   \centering
\hspace*{-1cm}   \includegraphics[width=0.3\textwidth]{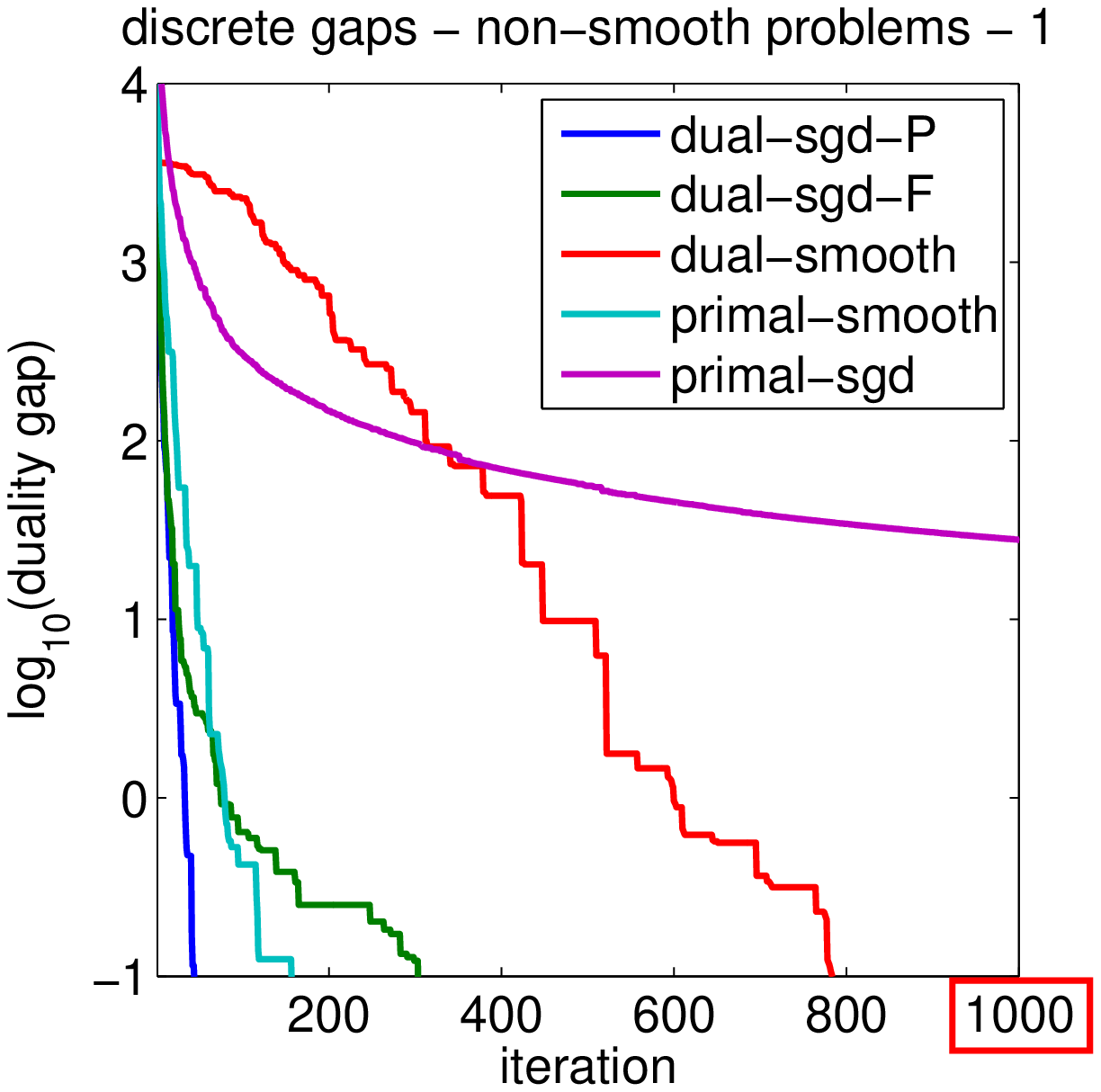}
   \includegraphics[width=0.3\textwidth]{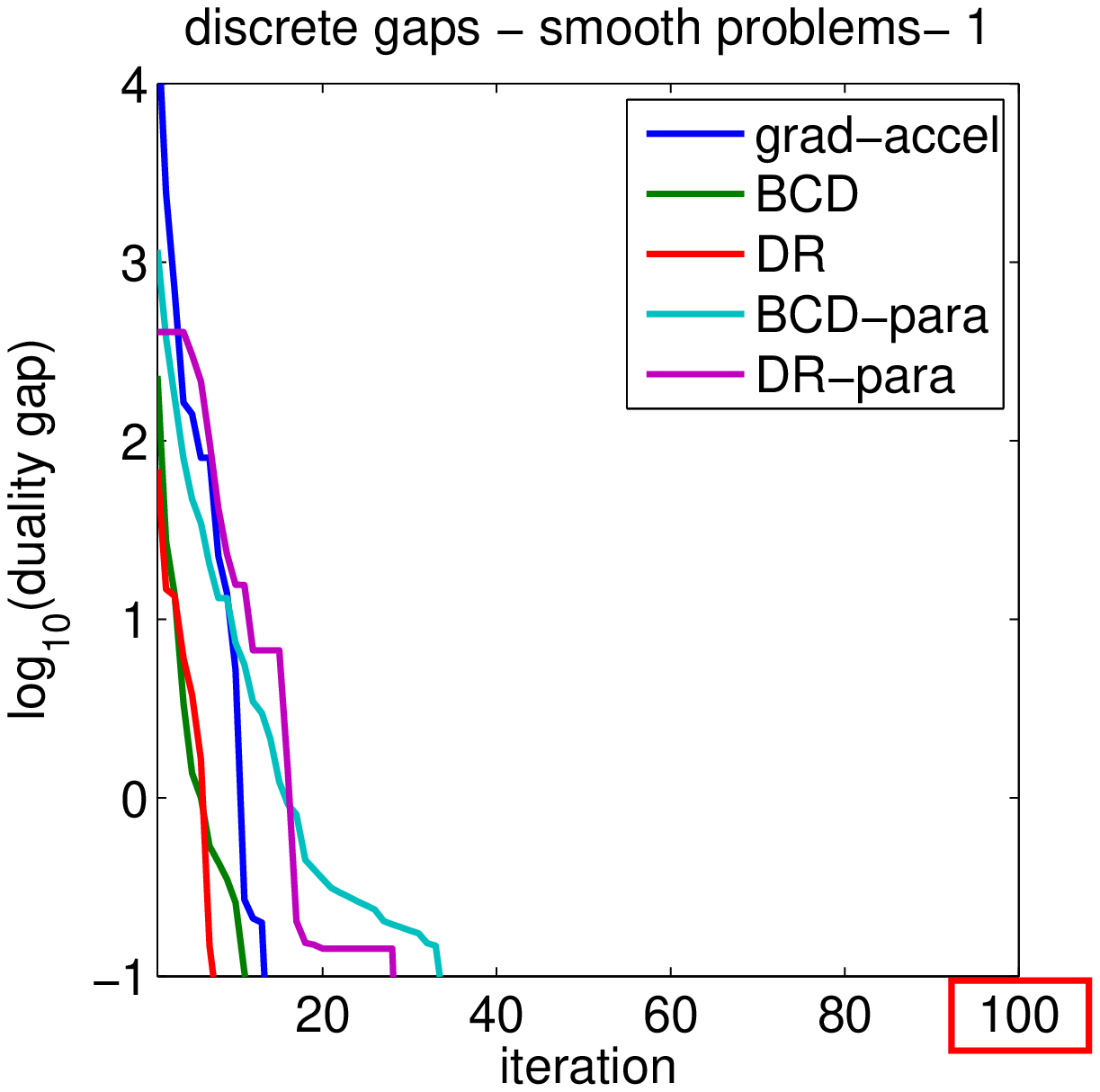}
   \includegraphics[width=0.3\textwidth]{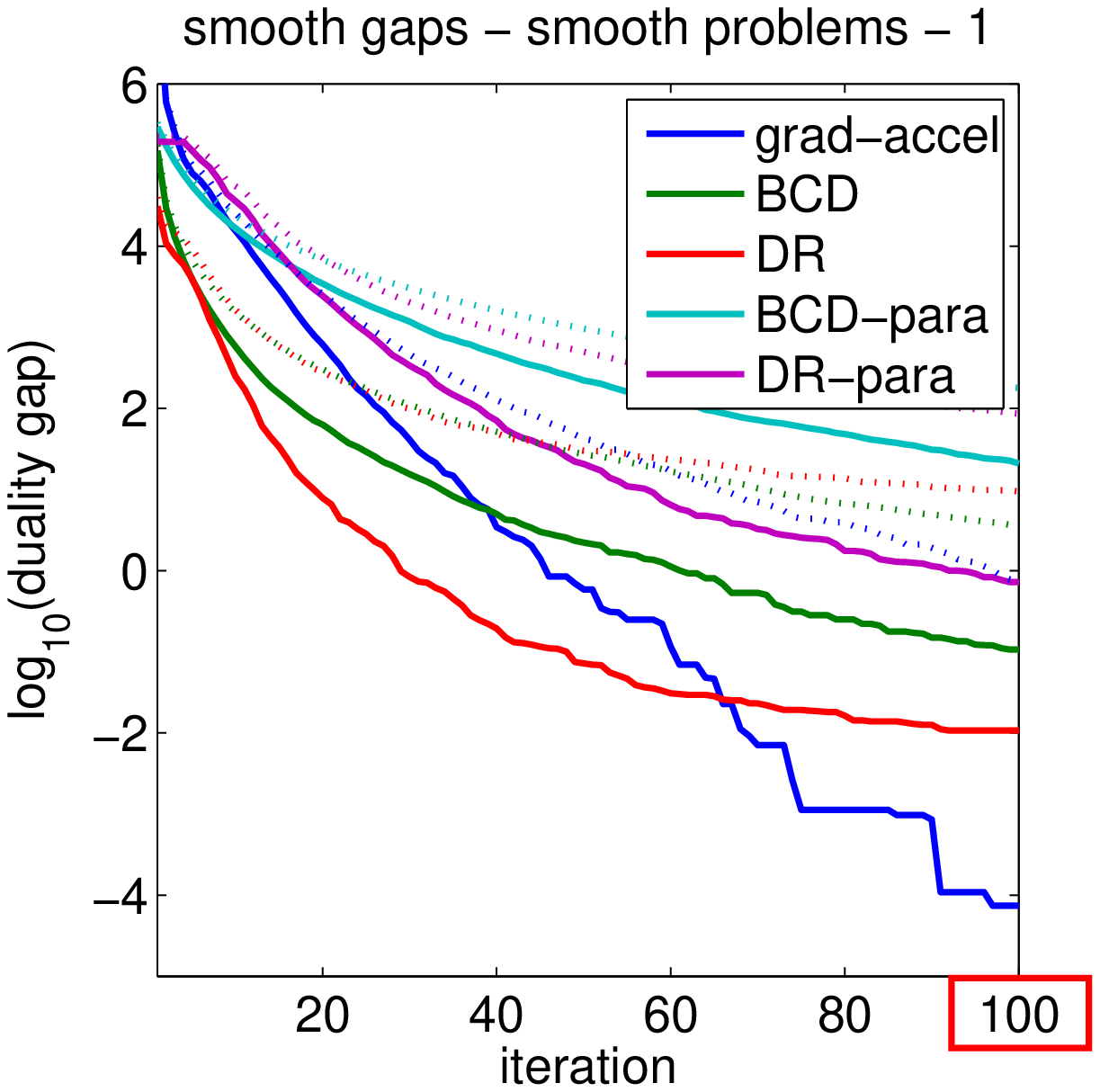}
   \hspace*{-1cm}   
 
  \hspace*{-1cm}   \includegraphics[width=0.3\textwidth]{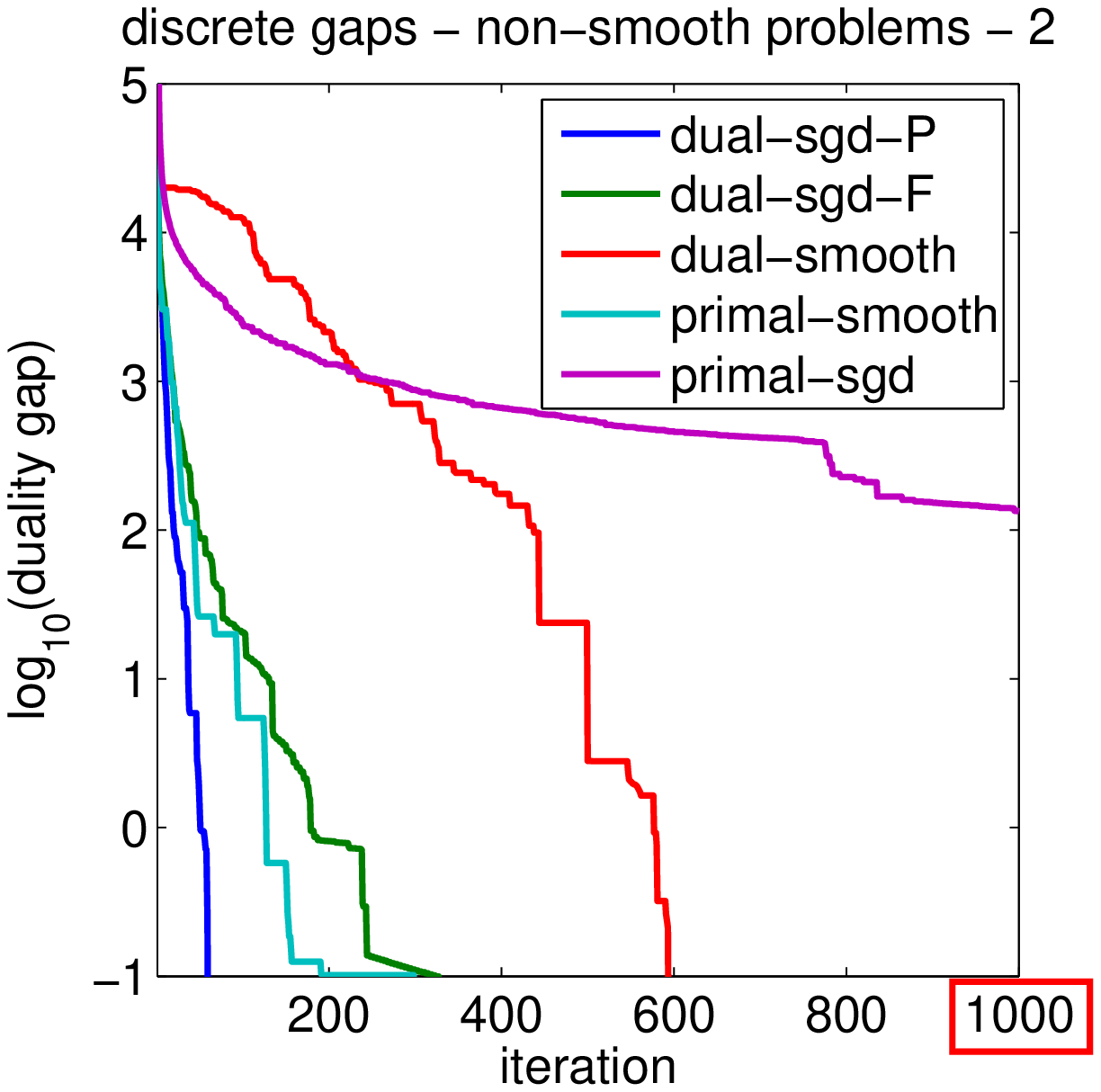}
    \includegraphics[width=0.3\textwidth]{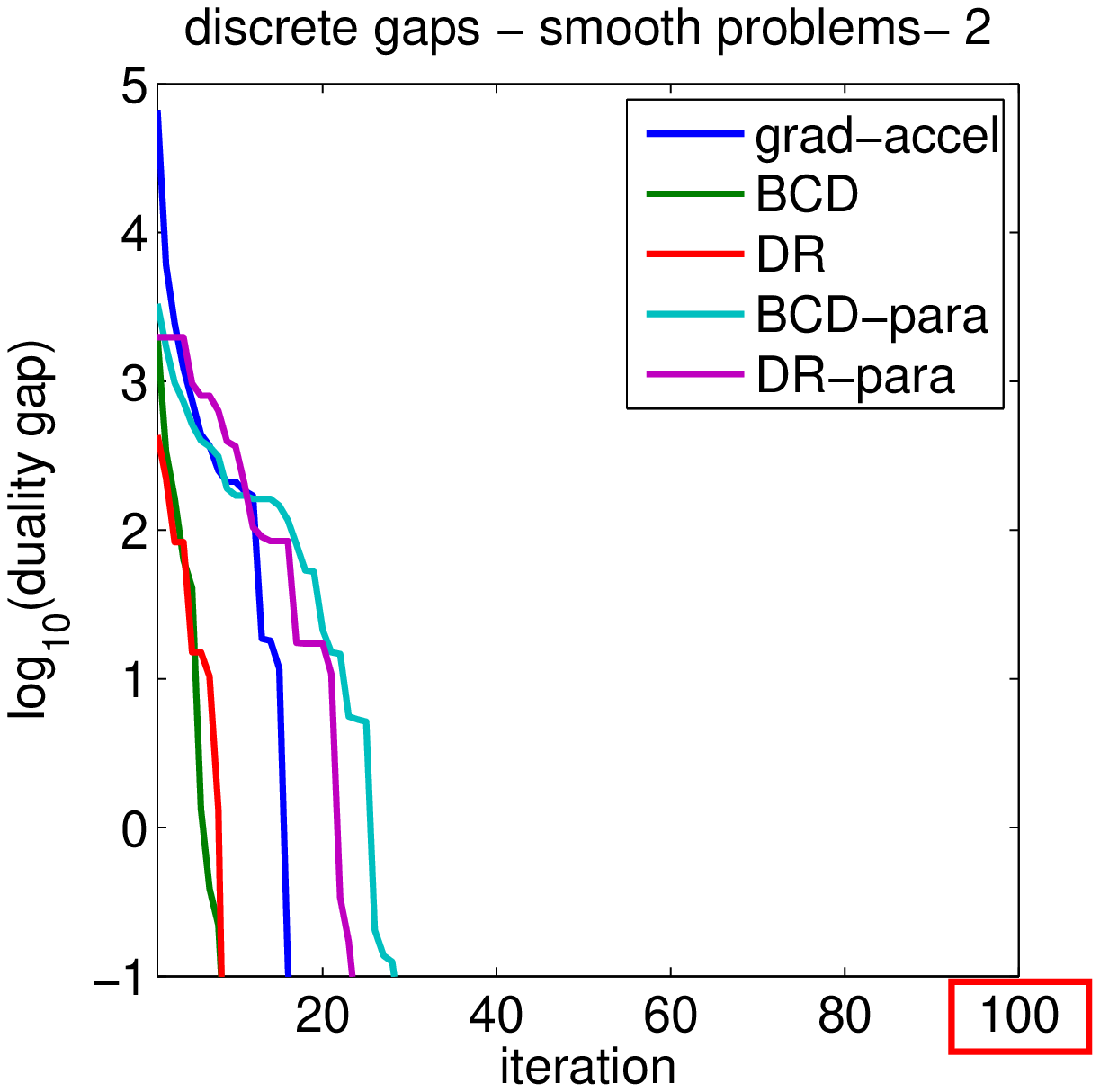}
    \includegraphics[width=0.3\textwidth]{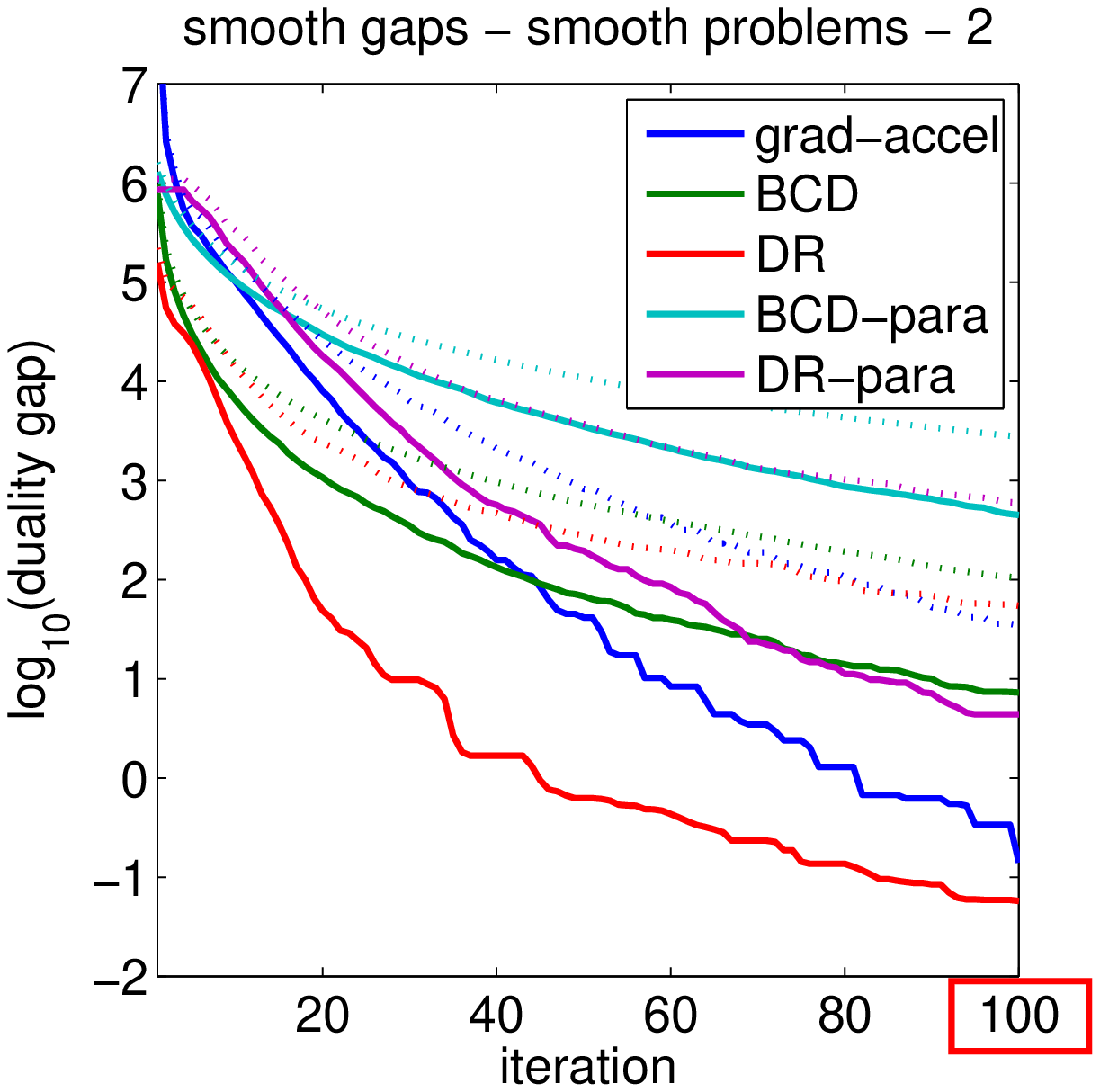}
    \hspace*{-1cm}   
 
  \hspace*{-1cm}   \includegraphics[width=0.3\textwidth]{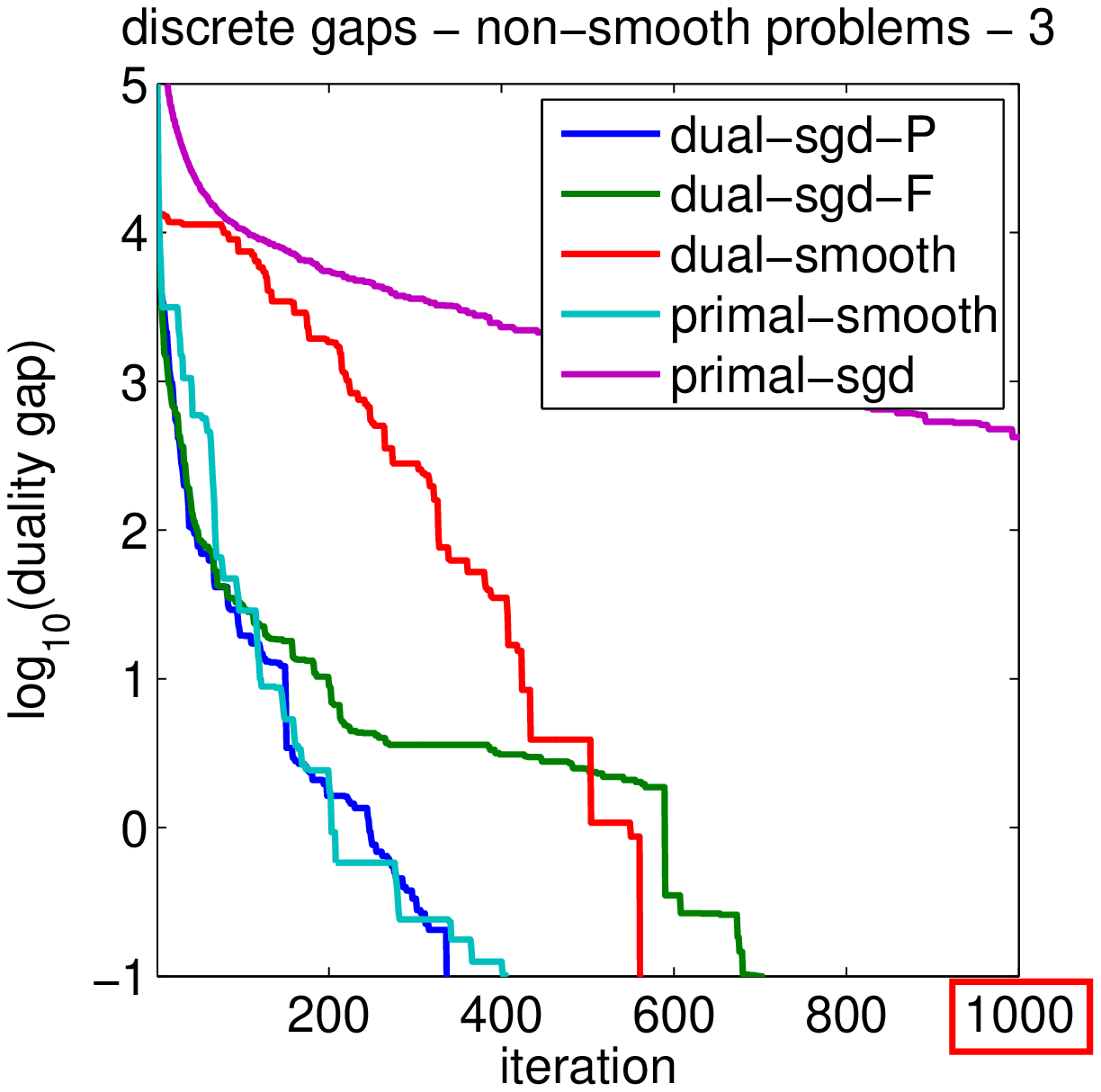}
    \includegraphics[width=0.3\textwidth]{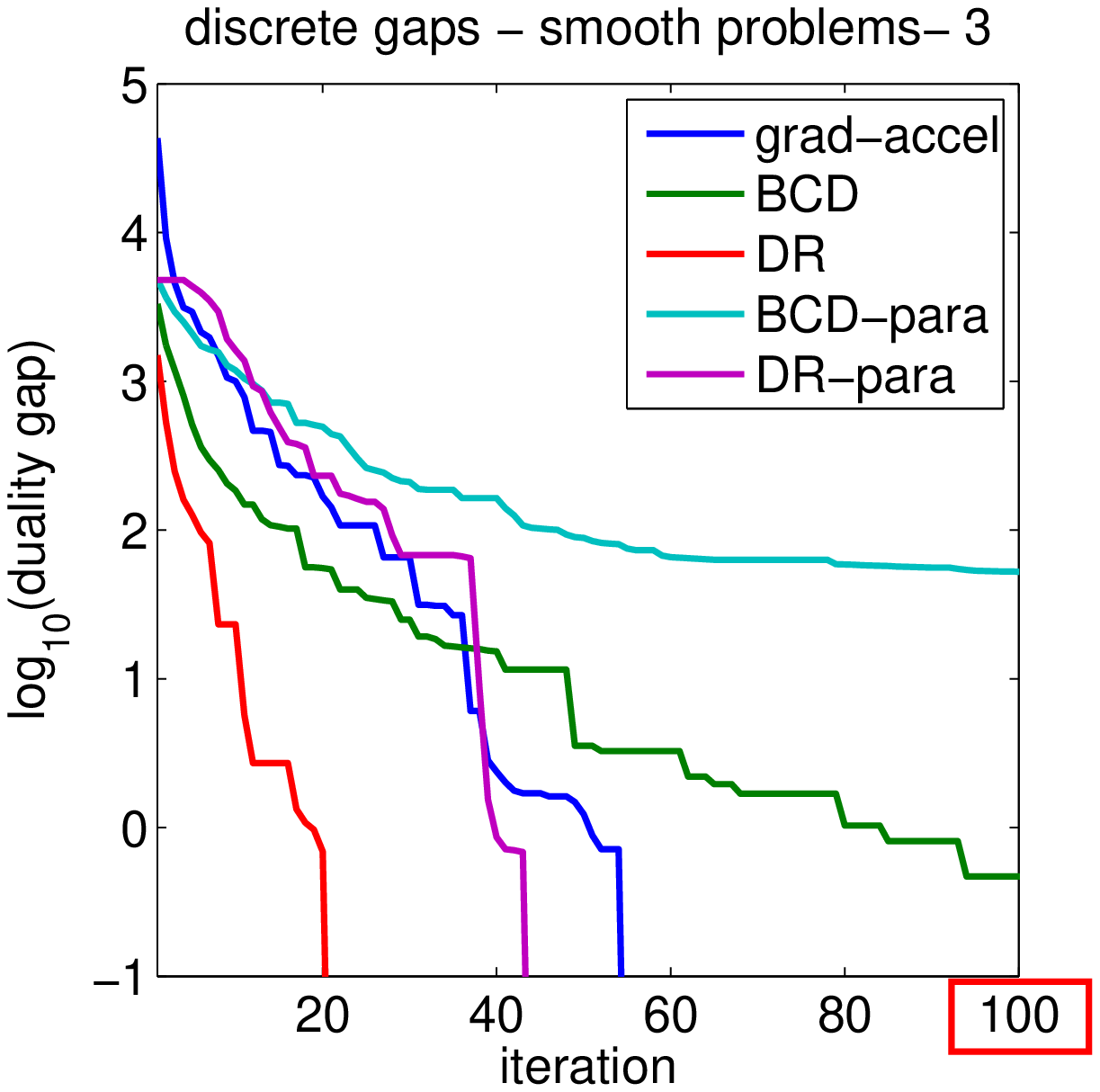}
    \includegraphics[width=0.3\textwidth]{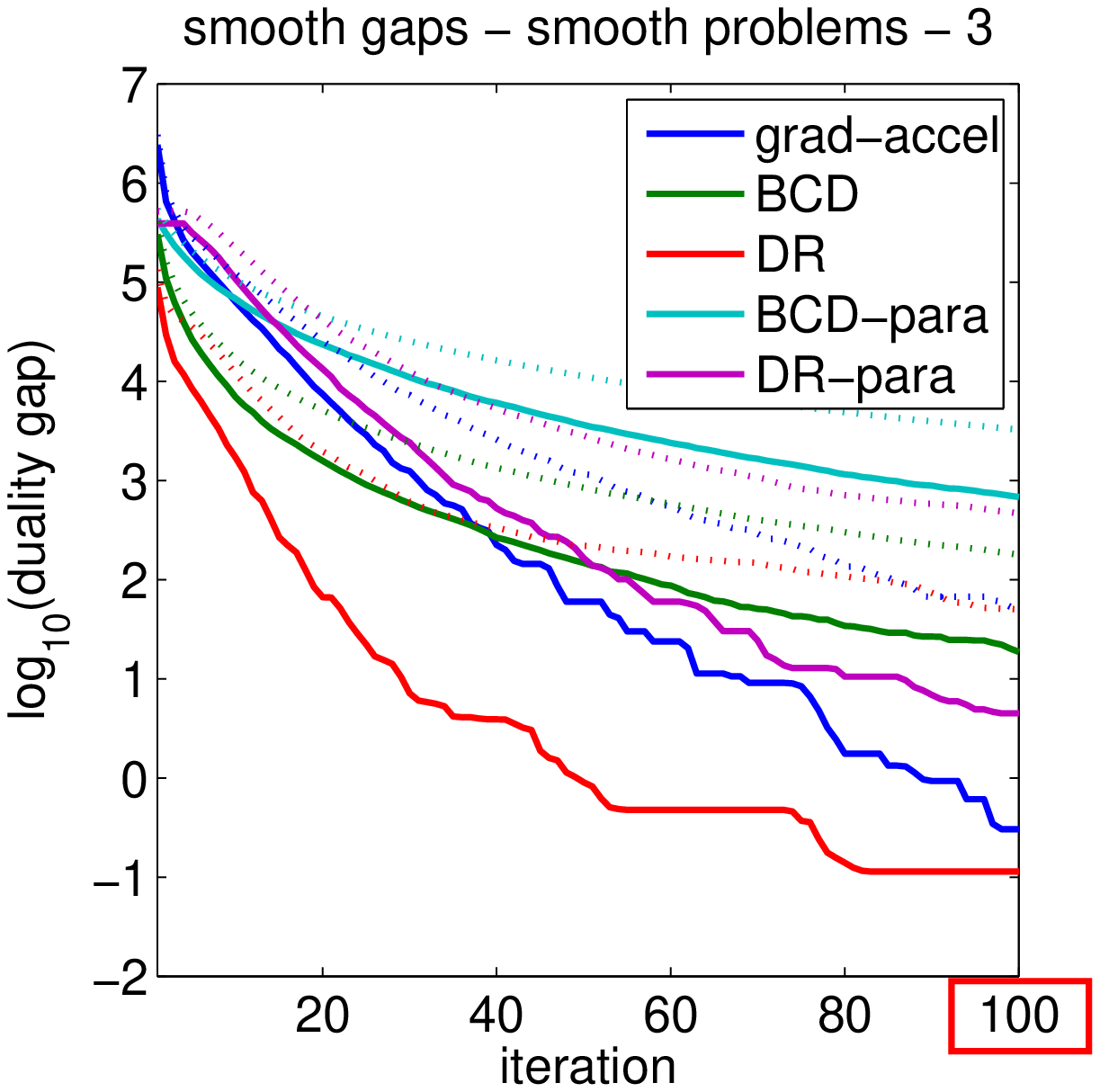}
    \hspace*{-1cm}

 \hspace*{-1cm}   \includegraphics[width=0.3\textwidth]{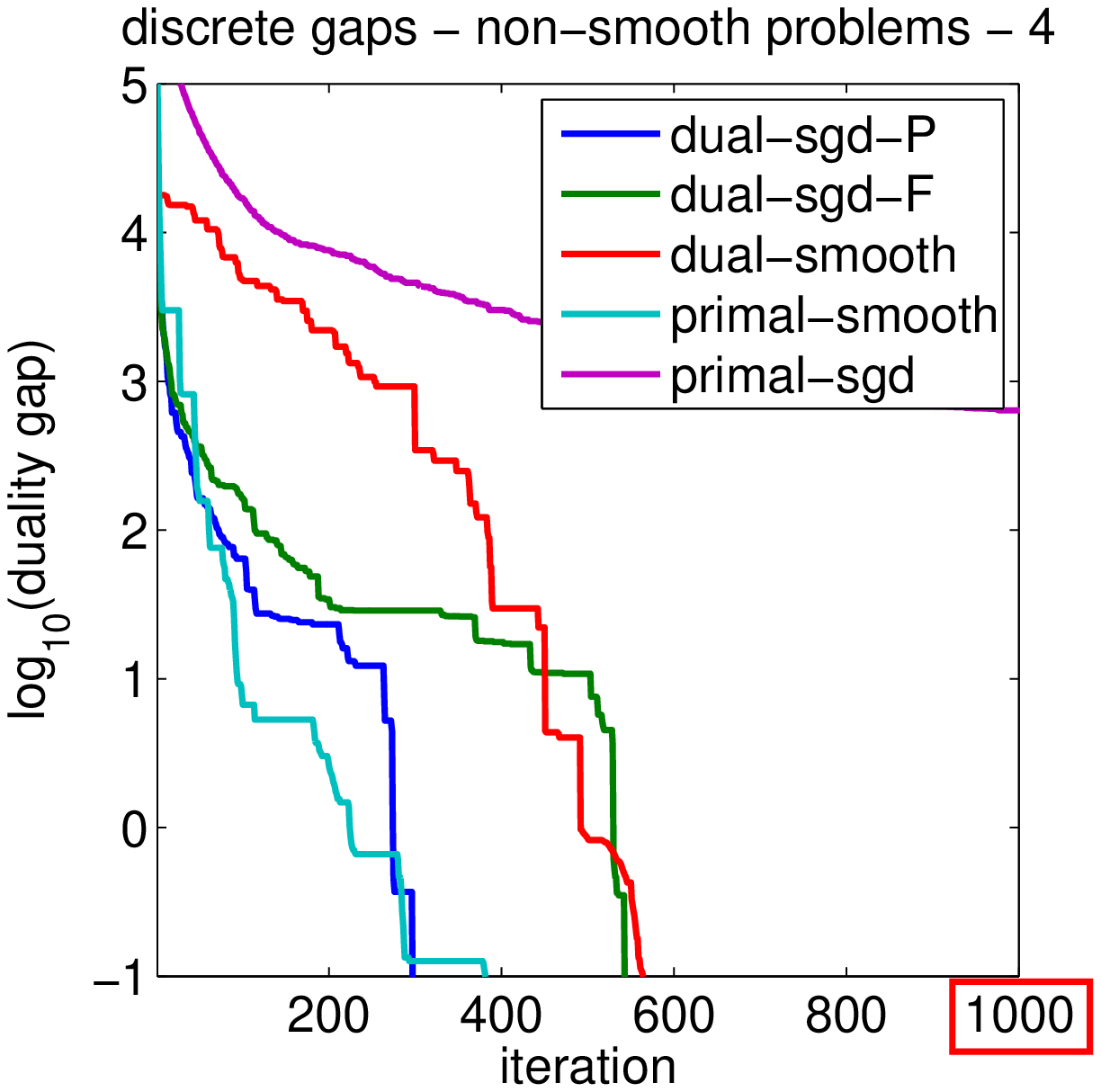}
   \includegraphics[width=0.3\textwidth]{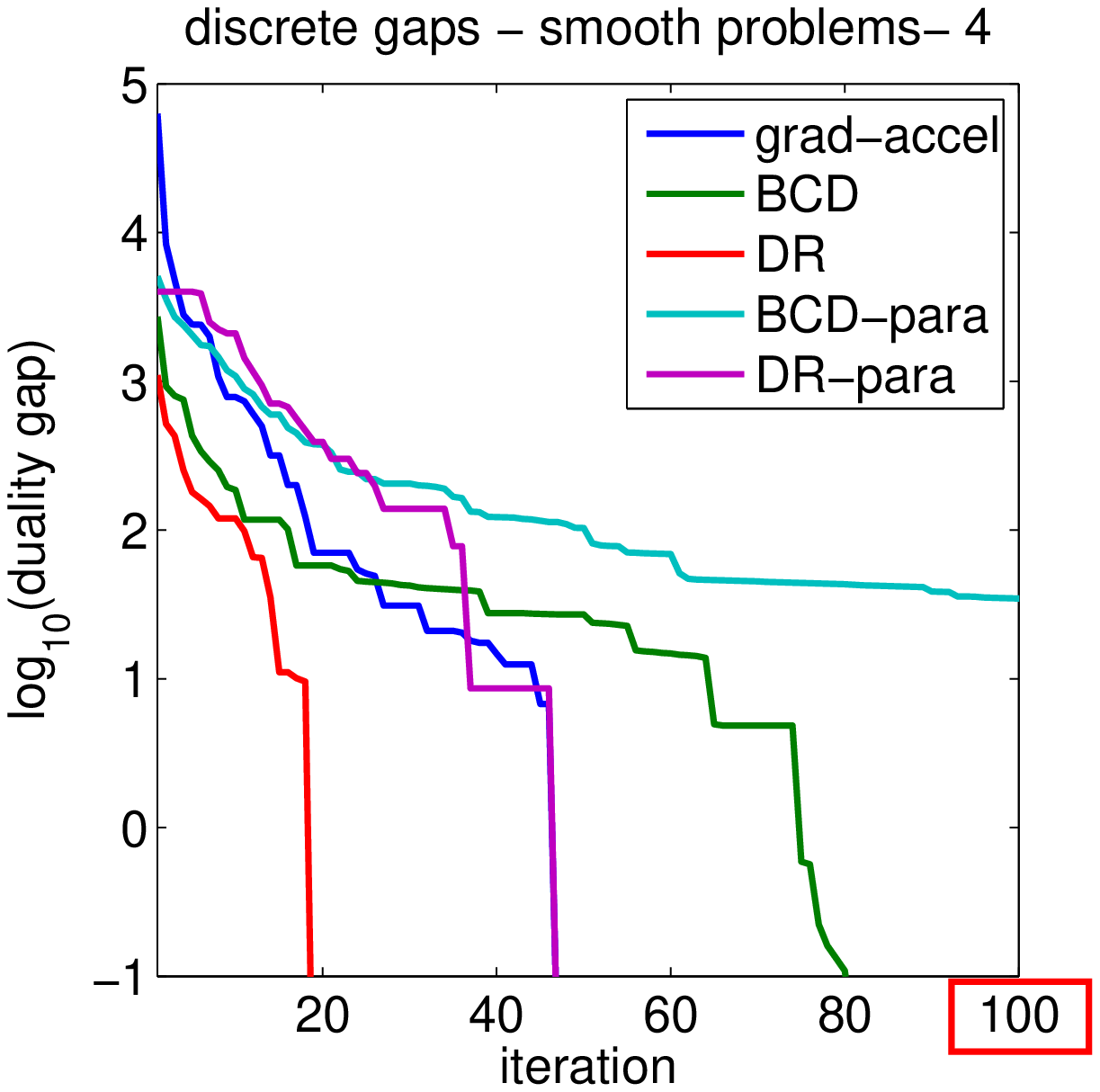}
   \includegraphics[width=0.3\textwidth]{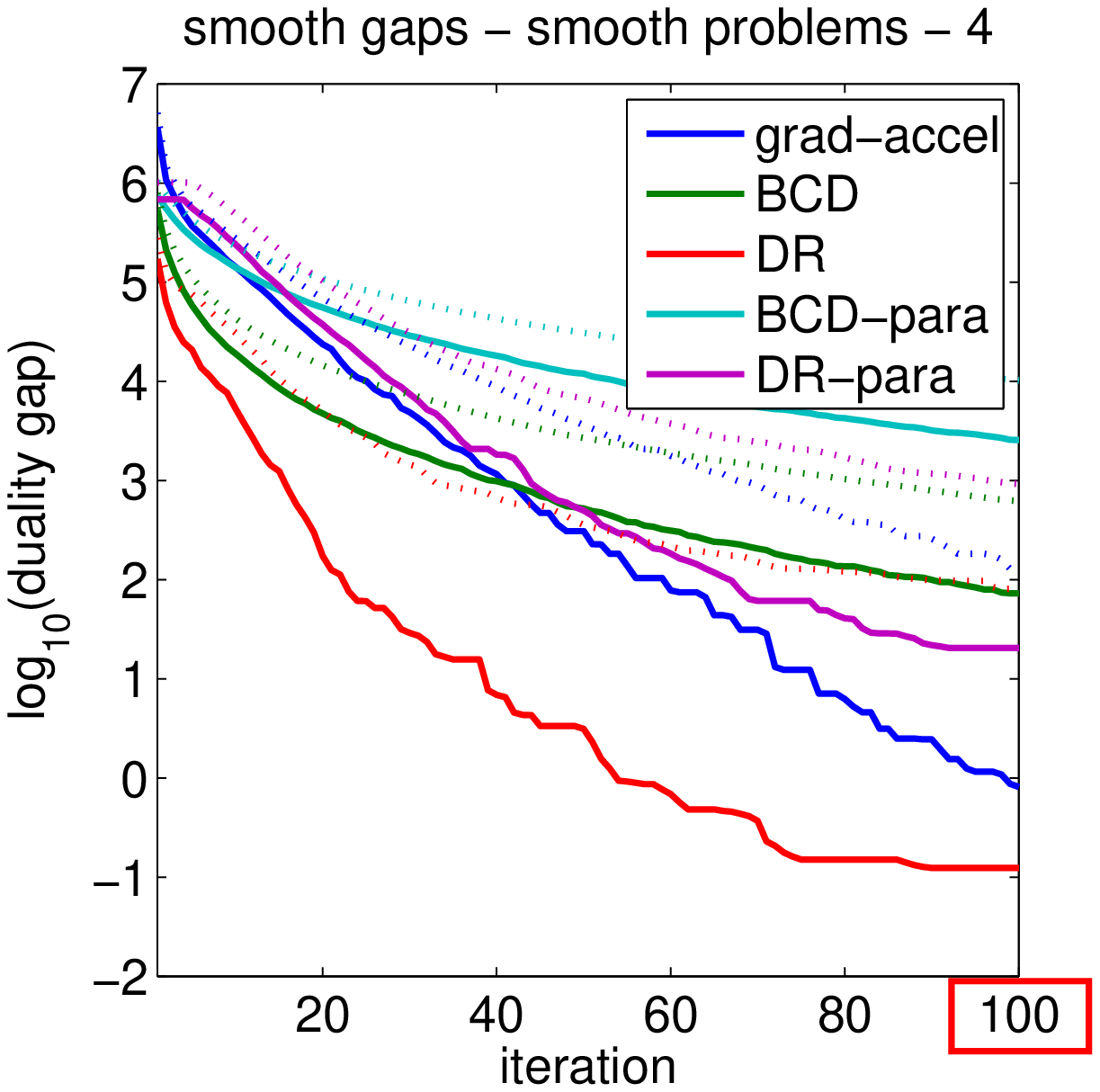}
   \hspace*{-1cm}   
   
   \vspace*{-.3cm}
   
   \caption{Comparison of convergence behaviors. Left: discrete duality gaps for various optimization schemes for the nonsmooth problem, from 1 to 1000 iterations.
   Middle: discrete duality gaps for various optimization schemes for the smooth problem, from 1 to 100 iterations. Right: corresponding continuous duality gaps. From top to bottom: four different images.}
   \label{fig:numiters}
 \end{figure}

 \begin{figure}
 
 \vspace{-2pt}
   \centering
\hspace*{-1cm}   
\includegraphics[width=0.25\textwidth]{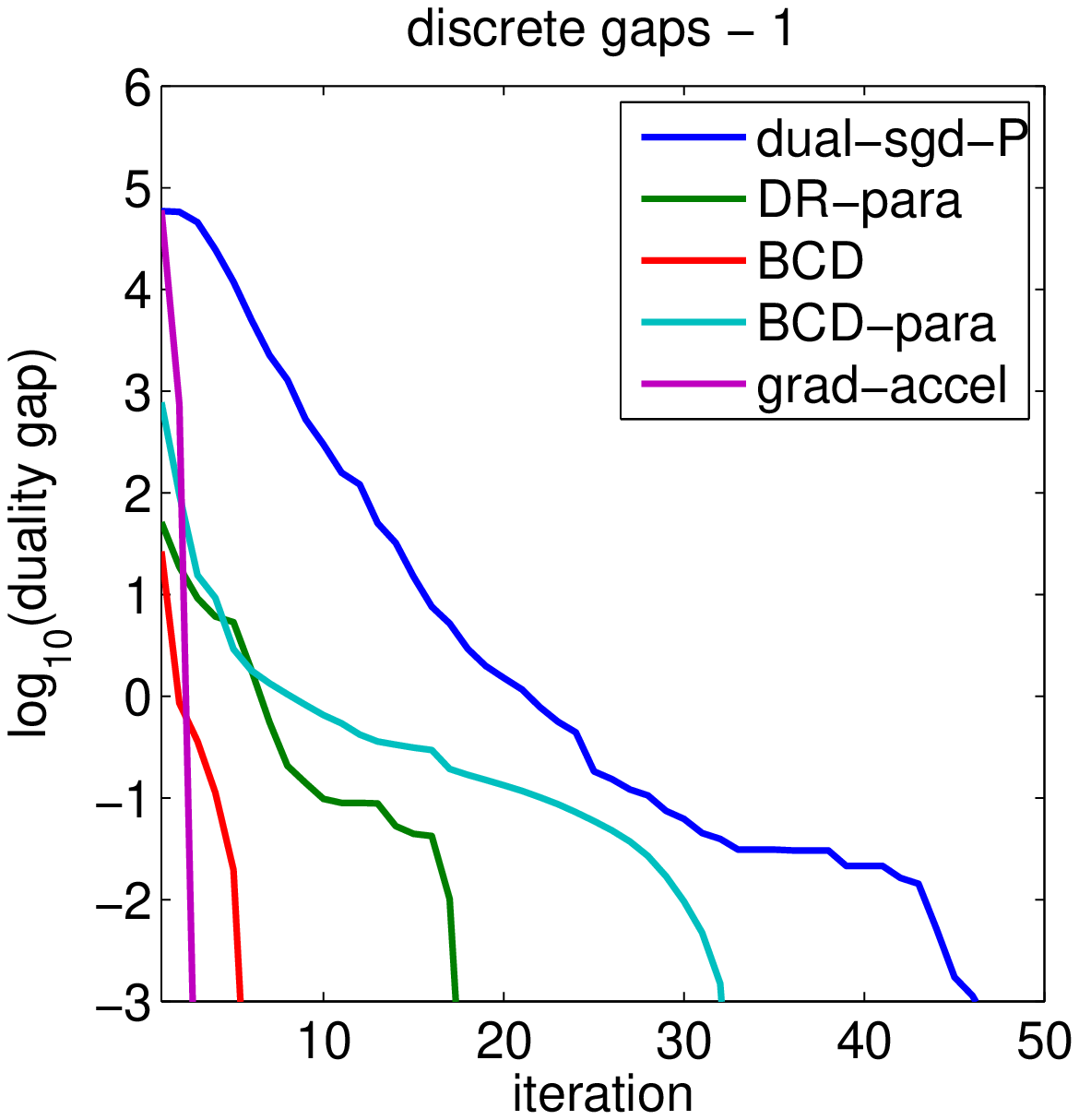}
   \includegraphics[width=0.25\textwidth]{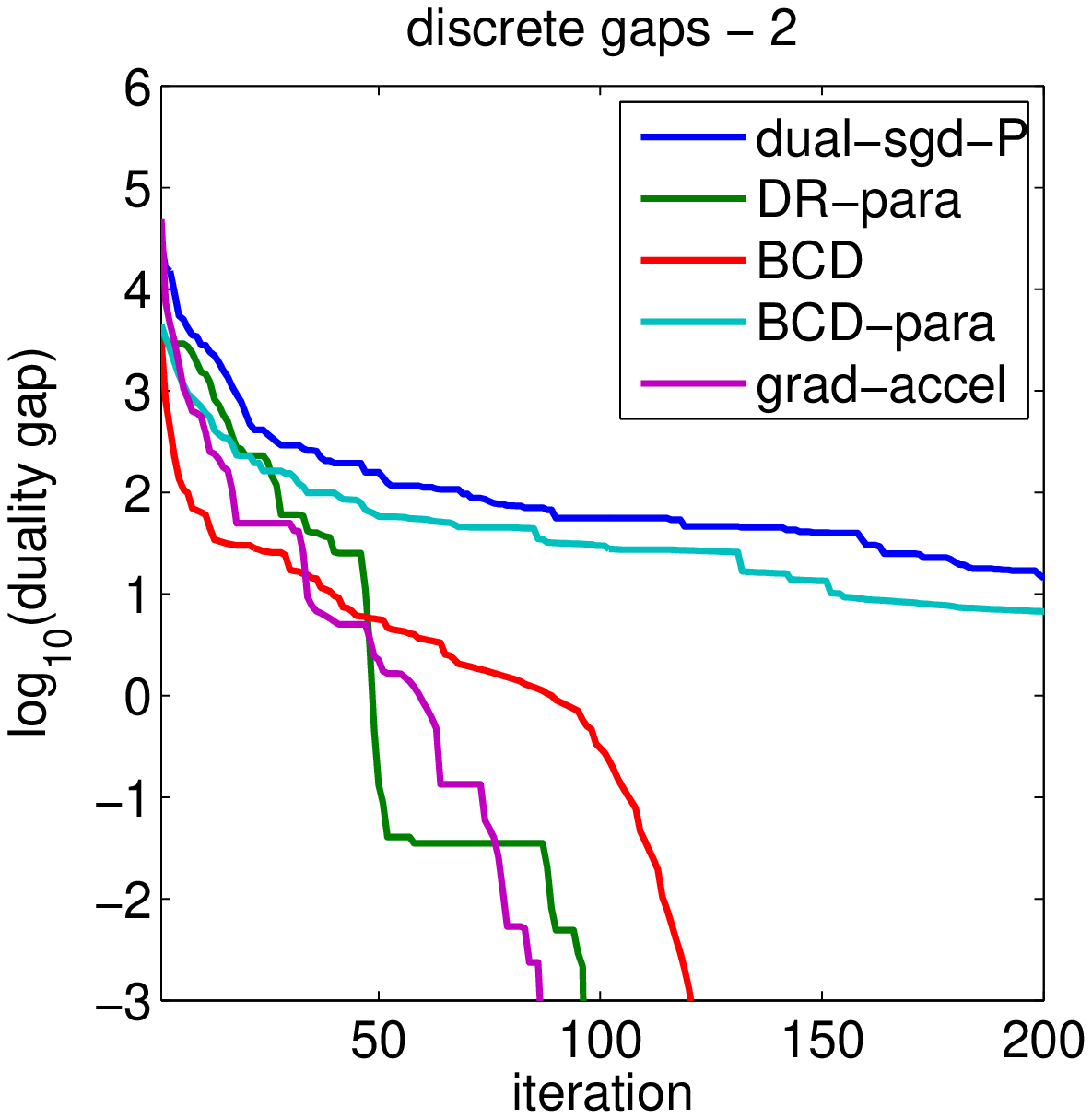}
   \includegraphics[width=0.25\textwidth]{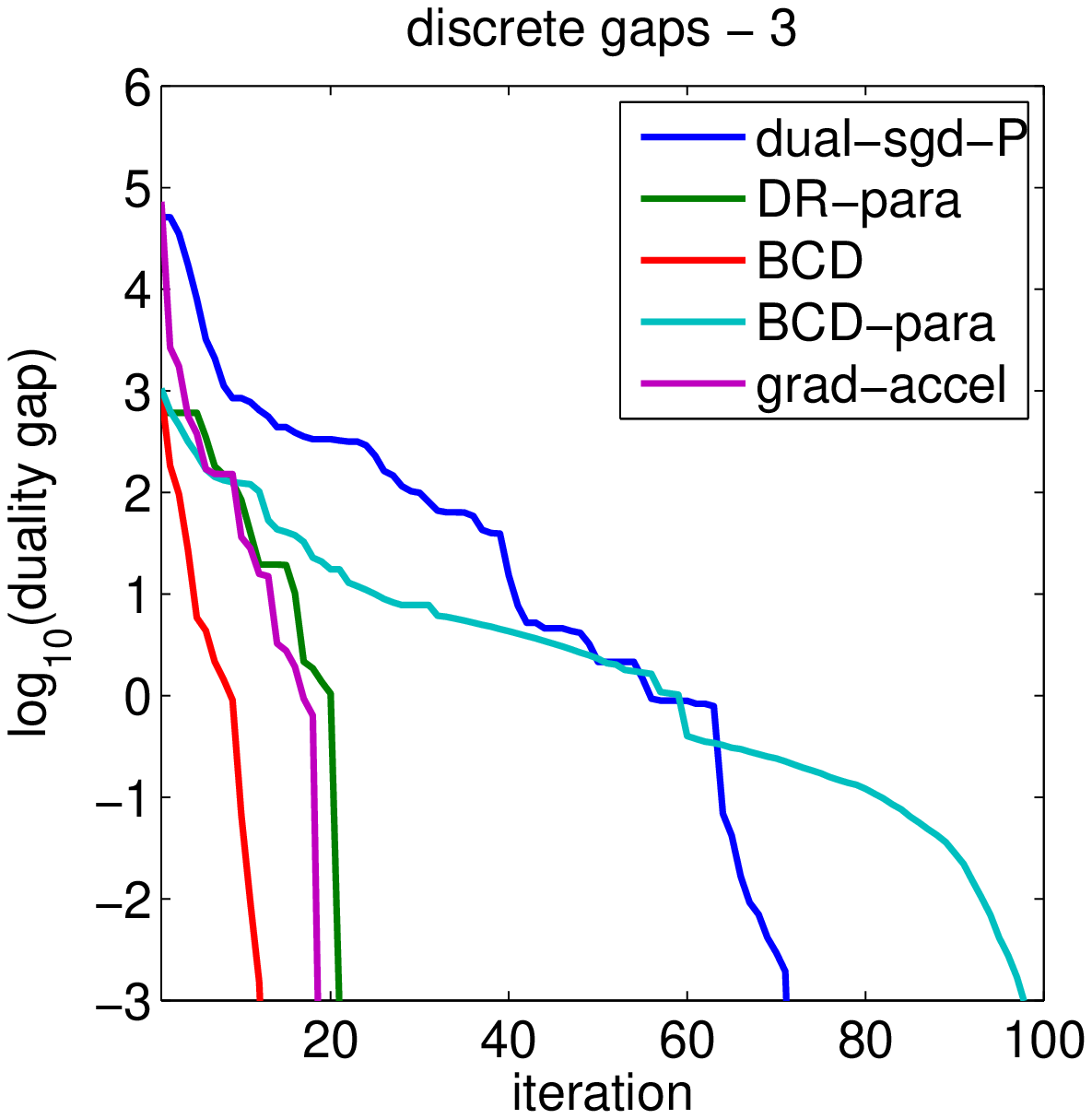}
   \includegraphics[width=0.25\textwidth]{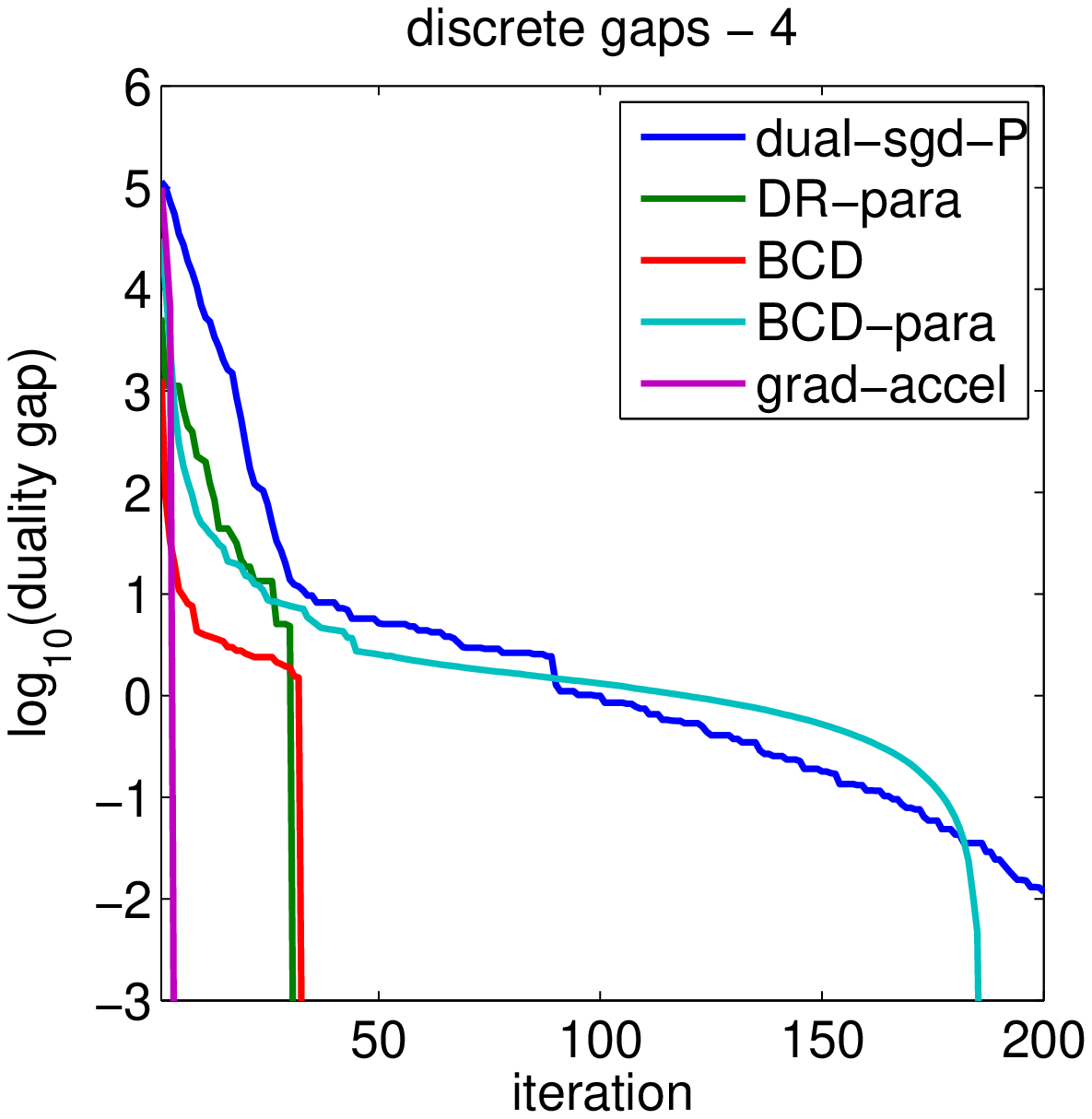}
   \hspace*{-1cm}  
   \vspace{-9pt}
   \caption{Convergence behavior for graph cut plus concave functions. 
   }
   \label{fig:numiters2} 
   \vspace*{-.15cm}
     
 \end{figure}

\paragraph{Parallel speedups} If we aim for parallel methods, then again DR outperforms BCD. Figure~\ref{fig.par} (right) shows the speedup gained from parallel processing for $r=2$. Using 8 cores, we obtain a 5-fold speed-up.

\begin{figure} 
  \centering
  \includegraphics[width=0.3\textwidth]{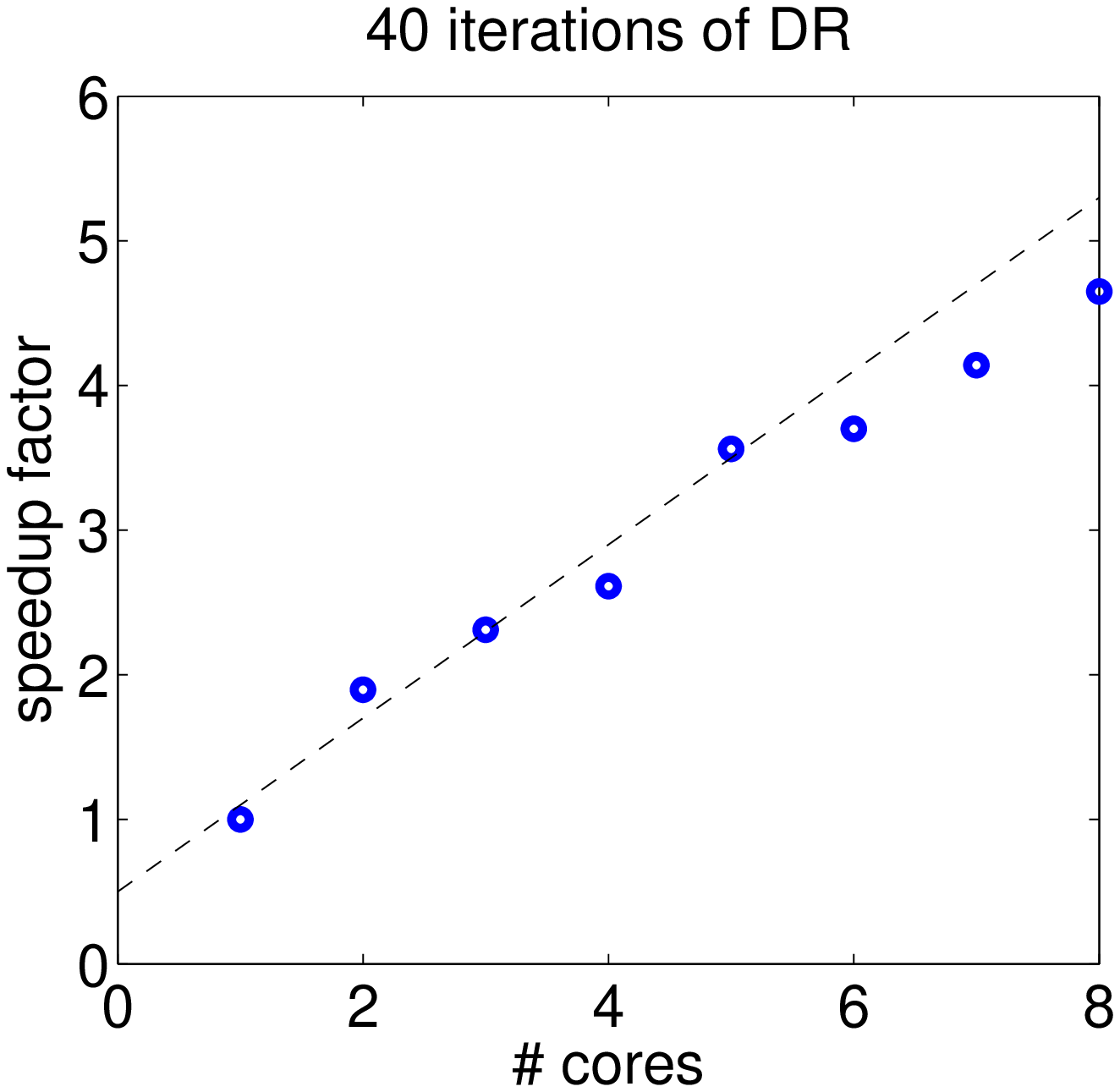}
  \includegraphics[width=0.3\textwidth]{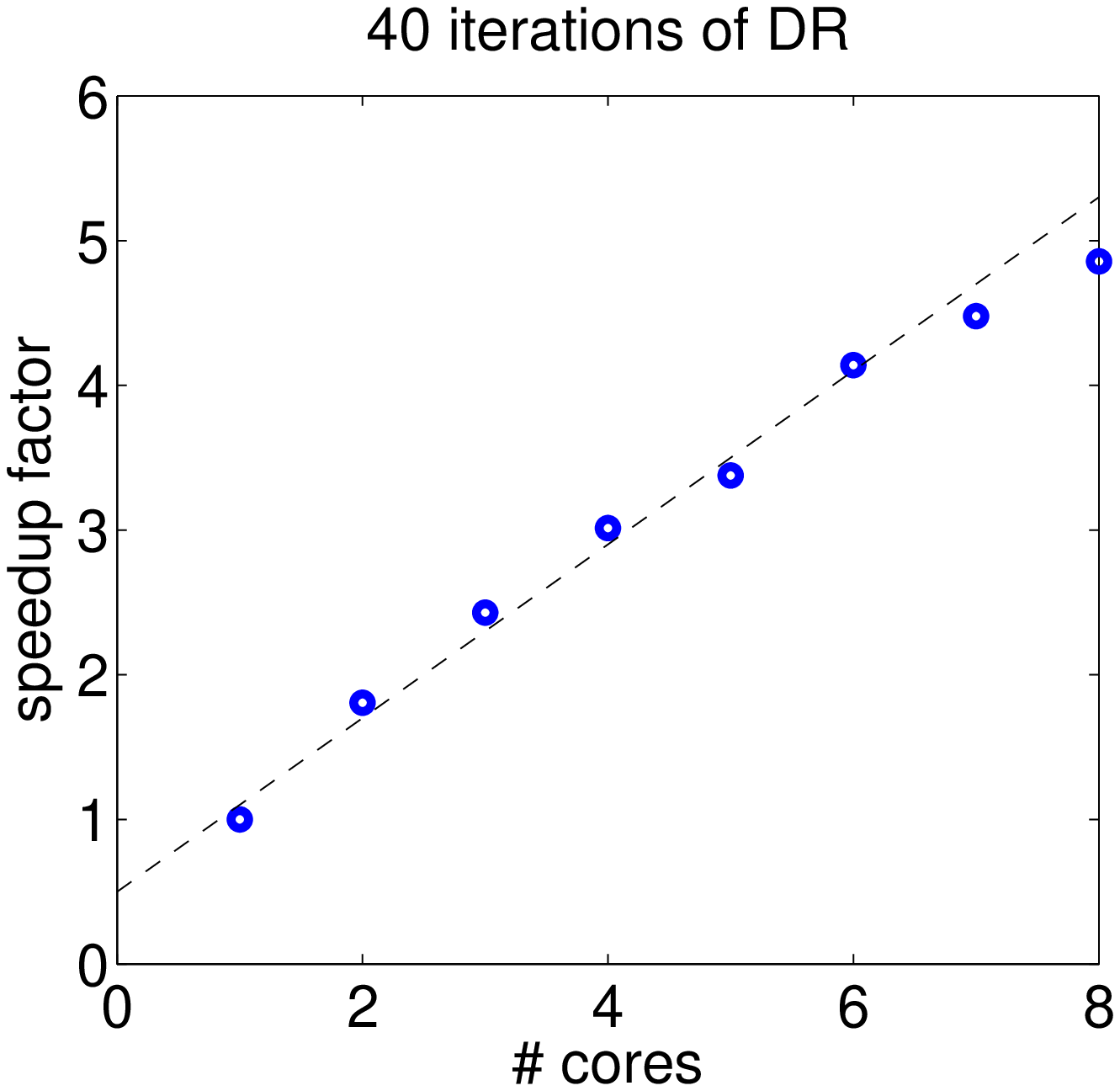}
  \caption{Speedup due to parallel processing for two instances.}
  \label{fig.par}
\end{figure}

\paragraph{Running time compared to graph cuts}
Table~\ref{tab:gcutcomp} shows the running times of our DR method (implemented in Matlab/C++) and the Maxflow code of \cite{boykov2001fast,boykov04,kolmogorov04} (using the wrapper \cite{Bagon2006}) for the four graph cut (segmentation) instances above on a MacBook Air with a 2 GHz Intel Core i7. The running times are averages over 5 repetitions. DR was run for 10, 10, 21, and 20 iterations, respectively.

DR is by a factor of 2-9 slower than the specialized code. Given that, as opposed to the combinatorial algorithm, DR solves the full regularization path, is parallelizable, generic and straightforwardly extends to a variety of functions, this is remarkable.

\begin{table}
  \centering
  \begin{tabular}[bla]{lrrrr}
    & Maxflow &    DR 1-thread & DR 2-thread & DR 4-thread\\ \hline
    image 1 & 0.39 & 1.61 (4.13) &   0.93 (2.39) & 0.65 (1.66) \\
    image 2 & 0.32 & 1.74 (5.45) & 0.99 (3.10) &  0.69  (2.16)  \\
    image 3 & 0.40 & 3.45 (8.61) &  1.93 (4.82) & 1.31  (3.27) \\
    image 4 & 0.38 & 3.38 (8.88) &   1.90 (5.00) &  1.29 (3.38) \\ \hline
    average & 0.37 & 2.55 & 1.44 & 0.98\\
  \end{tabular}
  \caption{Running times (in seconds) for the optimized C++ Maxflow code of \cite{boykov2001fast,boykov04,kolmogorov04} and our DR for graph cut using one or multiple threads. The last row is the average, and the numbers in parentheses indicate the factor relative to the Maxflow time.}
  \label{tab:gcutcomp}
\end{table}

In summary, our experiments suggest that projection methods can be extremely useful for solving the combinatorial submodular minimization problem. Of the tested methods, DR, cyclic BCD and accelerated gradient perform very well. For parallelism, applying DR on \eqref{eq:dualsmoothr2} converges much faster than BCD on the same problem.

\vspace*{-.25cm}
\section{Conclusion}
\vspace*{-.25cm}
We have presented a novel approach to submodular function minimization based on the equivalence with a best approximation problem. The use of reflection methods avoids any hyperparameters and reduce the number of iterations significantly, suggesting the suitability of reflection methods for combinatorial problems. Given the natural parallelization abilities of our approach, it would be interesting to perform detailed empirical comparisons with existing parallel implementations of graph cuts (e.g.,~\cite{shekhovtsov2011distributed}). Moreover, a generalization beyond submodular functions of the relationships between combinatorial optimization problems and convex problems would enable the application of our framework to other common situations such as multiple labels~(see, e.g.,~\cite{komodakis2011mrf}).

\vspace*{-.25cm}

\paragraph{Acknowledgments.} {\small This research was in part funded by the Office of Naval Research under contract/grant number N00014-11-1-0688,  by NSF CISE Expeditions award CCF-1139158, by  DARPA XData Award FA8750-12-2-0331, and the European Research Council (SIERRA project), as well as gifts from Amazon Web Services, Google, SAP, Blue Goji, Cisco, Clearstory Data, Cloudera, Ericsson, Facebook, General Electric, Hortonworks, Intel, Microsoft, NetApp, Oracle, Samsung, Splunk, VMware and Yahoo!.
We would like to thank Martin Jaggi, Simon Lacoste-Julien and Mark Schmidt for discussions.
}

\vspace*{-.25cm}

{

\bibliographystyle{plainnat}
\bibliography{refs2}
}

\appendix

\newcommand{\Cc}{\mathcal{C}}

\section{Derivations of Dual Problems}\label{sec:duals}

\subsection{Proof of Lemma 1}
\begin{proof}
  To derive the non-smooth dual problem, we
  follow~\cite{komodakis2011mrf} and use Lagrangian duality:
  \begin{eqnarray*}
    \min_{x \in [0,1]^n} f(x)  
    & = & \min_{x \in [0,1]^n} \nlsum_{j=1}^r f_j(x) 
    = \min_{x_1,\dots,x_r \in [0,1]^n }   \nlsum_{j=1}^r f_j(x_j) \mbox{ such that } x_1 = \dots = x_{r} \\
    & = & \min_{x \in \rb^n, \ x_1,\dots,x_r \in [0,1]^n } \max_{(\lambda_j)} \nlsum_{j=1}^r f_j(x_j) + \nlsum_{j=1}^r \lambda_j^\top ( x - x_j)\\
    & = &  \max_{\sum_{j=1}^r \lambda_j = 0} \nlsum_{j=1}^r \min_{x_j \in [0,1]^n}
    \big\{
    f_j(x_j)-   \lambda_j^\top x_j  \big\} \\
    & = &  \max_{\sum_{j=1}^r \lambda_j = 0} \nlsum_{j=1}^r \max_{y_j \in B(F_j)}
    (y_j  - \lambda_j)_-(V) =   \max_{\sum_{j=1}^r \lambda_j = 0} \nlsum_{j=1}^r g_j(\lambda_j), 
  \end{eqnarray*}
  where $g_j(\lambda_j) = \min_{A \subset V} F_j(A) - \lambda_j(A)$ is
  a nonsmooth concave function, which may be computed efficiently through submodular function minimization.
\end{proof}

\subsection{Proof of Lemma 2}
\begin{proof}
  The proof follows a similar saddle-point approach.
\begin{eqnarray}
\nonumber\min_{x \in \rb^n} f(x)  + \half \|x\|_2^2
& = & \min_{x \in \rb^n} \nlsum_{j=1}^r f_j(x)  + \half \|x\|_2^2 \\
\nonumber & = & \min_{x_1,\dots,x_r \in \rb^n }   \nlsum_{j=1}^r \Big\{  f_j(x_j) 
+ \frac{1}{2r} \| x_j\|_2^2 \Big\}\ \mbox{ such that } x_1 = \dots = x_{r} \\
\nonumber & = & \min_{x \in \rb^n, \  x_1,\dots,x_r \in \rb^n } \max_{\lambda_j} \nlsum_{j=1}^r \bigl\{  f_j(x_j)  + \frac{1}{2r} \| x_j\|_2^2 +  \lambda_j^\top ( x - x_j) \bigr\}\\
\nonumber & = &  \max_{\sum_{j=1}^r \lambda_j = 0} \nlsum_{j=1}^r \min_{x_j \in \rb^n}
\Big\{
 f_j(x_j)-    \lambda_j ^\top x_j  + \tfrac{1}{2r} \| x_j\|_2^2 \Big\} \\
 \nonumber & = &  \max_{\sum_{j=1}^r \lambda_j = 0} \nlsum_{j=1}^r \min_{x_j \in \rb^n}
\Big\{
\max_{ y_j \in B(F_j)} x_j^\top y_j -    \lambda_j ^\top x_j  + \tfrac{1}{2r} \| x_j\|_2^2 \Big\} \\
\nonumber & = &  \max_{\sum_{j=1}^r \lambda_j = 0} \nlsum_{j=1}^r \max_{y_j \in B(F_j)}
-\tfrac{r}{2} \| y_j   - \lambda_j\|_2^2
\\
\label{eq:dualsmooth2}  & = &  
\max_{\sum_{j=1}^r \lambda_j = 0} \max_{y_j \in B(F_j)}
-\frac{r}{2} \sum_{j=1}^r  \| y_j  - \lambda_j\|_2^2.
\end{eqnarray}
Writing \eqref{eq:dualsmooth2} as a minimization problem and ignoring constants completes the proof.
\end{proof}

\section{Divide-and-conquer algorithm for parametric submodular minimization}\label{sec:divconquer}

\subsection{Description of the algorithm}

The optimal solution $x^*$ of our proximal problem $\min_{x \in \rb^n} f(x) + \|x\|^2$ indicates the minimizers of $F(S) - \lambda|S|$ for all $\lambda \in \mathbb{R}$. Those minimizers form a chain $S_\emptyset \subset S_1 \subset \ldots \subset S_k = V$.
The solutions are the level sets of the optimal solution $x^*$.

Here, we extend the approach of~\citet{tarjan2006balancing} for parametric max-flow to all submodular functions and all monotone strictly convex functions beyond the square functions used in the main paper. More precisely, we consider a submodular function $F$ defined on $V = \{1,\dots,n\}$ and $n$ differentiable strictly convex functions $h_i$ such that their Fenchel-conjugates $h_i^\ast$ have full domain, for $i \in \{1,\dots,n\}$. The functions $h^\ast_i$ are then differentiable. We consider the following problem:
 \BEA
 \label{eq:fullprimal}
 \min_{x \in \rb^n} f(x) + \sum_{i=1}^n h_i(x_i)
 & = &  \min_{x \in \rb^n} \max_{y \in  B(F)} y^\top x   + \sum_{i=1}^n h_i(x_i) \\
 & = &  \max_{y \in  B(F)}   \min_{x \in \rb^n} y^\top x   + \sum_{i=1}^n h_i(x_i) \\
 \label{eq:fulldual}
 & = &  \max_{y \in  B(F)}   - \sum_{i=1}^n h_i^\ast(-y_i).
\EEA
The optimality conditions are
\begin{enumerate}
\item $y \in B(F)$,
\item $y^\top x = f(x)$,
\item $-y_i = h_i'(x_i) \Leftrightarrow x_i =
  (h_i^\ast)'(-y_i)$.
\end{enumerate}

Let $\tau(V)$ be the time for minimizing the submodular function $F(S) + a(S)$ on the ground set $V$ (for any $a \in \rb^n$). 
For our complexity analysis, we make the assumption that minimizing the (contracted) function $F^{S,a}(T) \triangleq F(S \cup T) - F(S) + a(T)$ on the smaller ground set $U \subseteq V\setminus S$ (for any $a \in \rb^n$, $S \subseteq V$, $U \subseteq V \backslash S$) takes time at most $\tfrac{|U|}{|V|}\tau(V)$. This is a reasonable assumption, because it essentially says that $\tau(V)$ grows at least linearly in the size of $V$. To our knowledge, even fast algorithms for special submodular functions take at least linear time.

We will also use the notation $F(S \mid T) \triangleq F(S \union T) - F(S)$. 
For recursions, we use the \emph{restriction} 
$F_S: 2^S \to \mathbb{R}$, $F_S(T) = F(T)$ of $F$ to $S$ and the \emph{contraction} $F^S: 2^{V\setminus S} \to \mathbb{R}$, $F^S(T) = F(T \mid S)$ of $F$ on $S$.

\begin{algorithm}
  \SetAlgoLined
  \SetKwFunction{FSplit}{SplitInterval}
  \DontPrintSemicolon
  \FSplit($\lambda_{\min}$, $\lambda_{\max}$, $V$, $F$, $i$)\;
    \eIf{$i$ even}{
      \tcp{unbalanced split}
      $\lambda \gets \argmin_{\lambda} \sum_i h_i(\lambda) - \lambda F(V)$\;
      $A \gets \argmin_{T \subseteq V} F(T) + \sum_{i \in T}h'_i(\lambda)$\;
      \If{$S = \emptyset$ or $S = V$}{
        return $x = \lambda \mathbf{1}_V$\;
      }
    }{
      \tcp{balanced split}
      $\lambda \gets (\lambda_{\min} + \lambda_{\max})/2$\;
      $S \gets \argmin_{T \subseteq V} F(T) + \sum_{i \in T}h'_i(\lambda)$\;
      \If{$S = \emptyset$}{
        $x \gets$ \FSplit($\lambda_{\min}$, $\lambda$, $V$, $F$, $i+1$)\;
        return $x$\;
      }
      \If{$S = V$}{
        $x \gets$ \FSplit($\lambda$, $\lambda_{\max}$, $V$, $F$, $i+1$)\;
        return $x$\;}
    }
    \tcp{$S \neq \emptyset$ and $S \neq V$}
    $x_S \gets$ \FSplit($\lambda_{\min}$, $\lambda$, $S$, $F^A$, $i+1$)\;
    $x_{V\setminus S} \gets$ \FSplit($\lambda$, $\lambda_{\max}$, $V \setminus S$, $F_S$, $i+1$)\;
    return $[x_S, x_{V \setminus S}]$\;
  \caption{Recursive Divide-and-Conquer}
  \label{alg:divconquer}
\end{algorithm}

Algorithm~\ref{alg:divconquer} is a divide-and-conquer algorithm. In each recursive call, it takes an interval $[\lambda_{\min}, \lambda_{\max}]$ in which all components of the optimal solution lie and either 
(a) shortens the search interval for any break point, (b) finds the optimal (constant) value of $x$ on a range of elements, or (c) recursively splits the problem into a set $S$ and $V\setminus S$ with corresponding ranges for the values of $x^*$ and finds the optimal values of $x$ on the two subsets.

\subsection{Review of related results}
The goal of this appendix is to show Proposition~\ref{lem:depth} below. We first start by reviewing existing results regarding separable problems on the base polyhedron (see \cite{bach2011learning} for details).

It is known that if $y \in B(F)$, then $y_k \in  \big[  F(V)- F(V \backslash \{k\}) , F( \{k \}) \big] $; thus, the optimal solution $x$ is such that
$x_k \in \big[ (h_k^\ast)'( - F( \{k \}) ),   (h_k^\ast)'( F(V \backslash \{k\}) - F(V) ) \big].$
We therefore set the initial search range to $$\lambda_{\min} =   \min_{k \in V} \  (h_k^\ast)'( - F( \{k \}) )\quad \mbox{ and }\quad \lambda_{\max} = \max_{k \in V}\  (h_k^\ast)'( F(V \backslash \{k\}) - F(V) ) .$$

The algorithm relies on the following facts (see \cite{bach2011learning} for a proof). For all three propositions, we assume that the $h_i$ are strictly convex, continuously differentiable functions on $\mathbb{R}$ such that 
$\sup_{\lambda \in \rb} h'(\lambda) = +\infty$ and $\inf_{\lambda \in \rb} h'(\lambda) = -\infty$.
\begin{proposition}[Monotonicity of optimizing sets]\label{prop:monotonicity}
  Let $\alpha < \beta$ and $S^\alpha$ be any minimizer of $F(S) + h'(\alpha)(S)$ and $S^\beta$ any  minimizer of $F(S) + h'(\beta)(S)$. Then $S^{\beta} \subseteq S^{\alpha}$.
\end{proposition}

\begin{proposition}[Characterization of $x^*$]\label{prop:optimizer}
  The coordinates $x_j^*$ ($j \in V$) of the unique optimal solution $x^*$ of Problem~\ref{eq:fullprimal} are
  \begin{equation*}
    x^*_j = \max\{\lambda \mid j \in S^{\lambda}\},
  \end{equation*}
  where $S^{\lambda}$ is any minimizer of $F(S) + h'(\lambda)(S)$.
\end{proposition}

Propositions~\ref{prop:monotonicity} and \ref{prop:optimizer} imply that the level sets of $x^*$ form a chain $\emptyset = S_0 \subset S_1 \subset \ldots \subset S_k = V$ of maximal minimizers for the critical values of $\lambda$ (which are the entries of $x^*$). (Each $S_i = S^{\lambda}$ for some $\lambda = x^*_j$.)

\begin{proposition}[Splits]\label{prop:split}
  Let $T = S_i$ be a level set of $x^*$ and let $y \in \rb^T, z \in \rb^{V\setminus T}$ be the minimizers of the subproblems
  \begin{align*}
    y &= \argmin_{x} f_{T}(x) + \sum_{i \in T} h_i(x_i)\\
    z &= \argmin_{x} f^{T}(x) + \sum_{i \notin T} h_i(x_i)
  \end{align*}
  Then $x^*_j = y_j$ for $j \in T$ and  $x^*_j = z_j$ for $j \in V\setminus T$. 
\end{proposition}
The algorithm uses Proposition~\ref{prop:split} recursively.
\begin{proof}
  Let $\lambda$ be the value in $x^*$ defining $S_i = S^\lambda$.
  It is easy to see that the restriction $F_T$ and the contraction $F^T$ are both submodular. Hence, Propositions~\ref{prop:monotonicity} and \ref{prop:optimizer} hold for them. 

  Since the restriction on $T$ is equivalent to the original function for any $S \subseteq T$,  
  $F(S) + h(\lambda)(S) = F_T(S) + h_T(\lambda)(S)$ for any $S \subseteq T$. With this, Propositions~\ref{prop:monotonicity} and \ref{prop:optimizer} imply that for any $\alpha > \lambda$, $F(S) + h(\lambda)(S) = F_T(S) + h_T(\lambda)(S)$ and therefore $x^*_j = y_j$ for $j \in T$.

  Similarly, for any $S \in V \setminus T$, it holds that $F(S \union T) + h(\lambda)(S \union T) = F^T(S) + (h'(\lambda))^T(S) + F(T) + h'(\lambda)(T)$. Due to the monotonicity property of the optimizing sets, $S^\alpha \supseteq T$ for all $\alpha < \lambda$, and therefore the maximal minimizer $U^\alpha$ of $F^T(S) + (h'(\alpha))^T(S)$ satisfies $U^\alpha \union T = S^\alpha$ (the terms $F(T) + h'(\lambda)(T)$ are constant with respect to $U$). Hence Poposition \ref{prop:optimizer} implies that $x^*_j = z_j$ for $j \in V\setminus T$.
\end{proof}

These propositions imply that there is a set of at most $n$ values of $\lambda=\alpha$ that define the level sets~$S^\alpha$ of the optimal solution $x^*$. If we know these break point values, then we know $x^*$. Algorithm~\ref{alg:divconquer} interleaves an unbalanced split strategy that may split the seach interval in an unbalanced way but converges in $O(n)$ recursive calls, and a balanced split strategy that always halves the search intervals but is not finitely convergent.

\subsection{Proof of convergence}

We now prove the convergence rate for Algorithm~\ref{alg:divconquer}.

\begin{proposition}\label{lem:depth}
  The minimum of $f(x) + \sum_{i=1}^n h_i(x_i)$ may be obtained up to coordinate-wise accuracy $\epsilon$ within
  \begin{equation}
    O\left(\min\{n, \log \tfrac{1}{\epsilon} \}\right)
    \end{equation}
submodular function minimizations. If $h_i(x_i) = \half x_i^2$, then $\epsilon = \tfrac{\Delta_{\min}}{n^2\ell_0}$ is sufficient to recover the exact solution, where
$\Delta_{\min} = \min \{ |F(S \mid T)| \mid S \subseteq V \setminus T, F(S \mid T) \neq 0\}$ and $\ell_0$ is the length of the initial interval $[\lambda_{\min}, \lambda_{\max}]$.
\end{proposition}

\begin{proof} 
  The proof relies on Propositions~\ref{prop:monotonicity}, \ref{prop:optimizer} and \ref{prop:split}. 
 
  We first argue for the correctness of the balanced splitting strategy. Propositions~\ref{prop:monotonicity} and \ref{prop:optimizer} imply that for any $\lambda \in \mathbb{R}$, if $S$ is a minimizer of $F(S) + h'(\lambda)(S)$, then the unique minimum of $f(x) + \sum_{i=1}^n h_i(x_i)$ satisfies that for all $k \in S$, $x_k \geqslant h'_k(\lambda)$ and  for all $k \in V \backslash S,  x_k \leqslant h'_k(\lambda)$.
  In particular, if $S=\varnothing$, then this means that for all $k \in V$,  $x_k \leqslant h'_k(\lambda)$. Similarly, if $S = V$, then for all $k \in V$,  $x_k \geqslant h'_k(\lambda)$. The limits of the interval are set accordingly.
  The correctness of the recursive call follows from Proposition~\ref{prop:split}.

  In each iteration, the size of the search interval $[\lambda_{\min},\lambda_{\max}]$ for any break point is at least halved. Hence, within $d$ recursions, the length of each interval is at most $2^{-d}\ell_0$.

  The choice of $\lambda$ in the unbalanced splitting strategy corresponds to solving a simplified version of the dual problem. Indeed, by convex duality, the following two problems are dual to each other:
  \begin{align}
    \label{eq:dualsmall}
    \max_{y}& - \sum_i h^*_i(-y_i)\quad \text{s.t. } y(V) = F(V)\\
    \label{eq:primalsmall}
    \min_{\lambda \in \mathbb{R}}& \sum_{i \in V}h_i(\lambda) - \lambda F(V).
  \end{align}
  Problem~\eqref{eq:dualsmall}  replaces the constraint that $y \in B(F)$ by $y(V) = F(V)$, dropping the constraint that $y(S) \leq F(S)$ for all $S \subseteq V$. Testing whether $y$ satisfies all constraints of \eqref{eq:fulldual}, i.e., $y \in B(F)$ is equivalent to testing whether $F(S) - y(S) \geq 0$. We do this implicitly by our choice of $\lambda$: 
  Convex duality implies that the the optimal solutions of Problems~\eqref{eq:dualsmall} and \eqref{eq:primalsmall} satisfy $y_i = -h'_i(\lambda)$. This holds in particular for the chosen (unique optimal) $\lambda$ in the algorithm. 

Let $T$ be a minimizer of $F(S) + h'(\lambda)(S) = F(S) - y(S)$. If $T = \emptyset$ or $T = V$, then $y \in B(F)$ and an optimal solution for the full dual problem~\eqref{eq:fulldual}. Hence, $y$ and $x = \lambda \mathbf{1}_V = (h^*)'(-y)$ form a primal/dual optimal pair for \eqref{eq:fulldual}.

If $\emptyset \subset T \subset V$ and $F(T)-y(T)<0$, then $y \notin B(F)$, and we perform a split with the same argumentation as above. This splitting strategy is exactly that of~\cite{groenevelt1991two,bach2011learning} and splits at most $n$ times. Hence, this strategy yields the global optimum (to machine precision) in the time of $O(n)$ times solving a submodular minimization on $V$. If $n$ is large, this may be computationally expensive.

If we only do balanced splits, we end up approaching the break points more and more closely (but typically never exactly). Unbalanced splits always find an exact break point, but with potentially little progress in reducing the intervals. Algorithm~\ref{alg:divconquer} thus interleaves both strategies where we store intervals of allowed values for subsets of  components of $A$. At step $d$ there are at most $\min \{n, 2^d\}$ different intervals (as there cannot be more intervals than elements of $V$). To split these intervals, submodular function minimization problems have to be solved on each of these intervals, with total complexity less than a single submodular function optimization problem on the full set. At each iteration, intervals corresponding to a singleton may be trivially completely solved, and components which are already found are discarded. Hence, at each recursive level, the total computation time is bounded above by $\tau(V)$.

While balanced splits always substantially shrink the intervals, they are not finitely convergent. Unbalanced splits converge after at most $n$ recursions. Following the argumentation of \citet{tarjan2006balancing}, who considered the special case of flows, alternating the two types of splits gives the best of both worlds: (a) all components are estimated up to precision  $ \frac{\ell_0}{2^{d/2}}$, and (b) the algorithm is finitely convergent, and will stop when $ \frac{
\ell_0 }{2^{d/2}}$ is less than the minimal distance between two different components of $x$.  

Finally, we adress the precision for the special case that $h_i(x_i) = \half x_i^2$ for all $i \in V$. If the interval lengths are smaller than the smallest gap between any two break points (components of $x^*$), then each interval contains at most one break point and the algorithm converges after at most two unbalanced splits. Hence, we here consider $\epsilon$ to be one half times the smallest gap between any two break points. Let $\emptyset = S_0 \subset S_1 \subset \ldots \subset S_k = V$ be the chain of level sets of $x^*$. By the optimality conditions discussed above for unbalanced splits, any constant part $T = S_i \setminus S_{i-1}$ of $x^*$ takes value $\lambda \mathbf{1} = -y_j \mathbf{1}$ ($j \in T$), where $y(T) = F_{S_{i-1}}(T)$, and hence 
\begin{equation}
  \lambda = -\frac{F(S_i\setminus S_{i-1}|S_{i-1})}{|S_i \setminus S_{i-1}|}.
\end{equation}
Therefore, the (absolute) difference between any two such values is loosely lower bounded by
\begin{align}
  \min_{i} \left| \frac{F(S_i\setminus S_{i-1}|S_{i-1})}{|S_i \setminus S_{i-1}|} - \frac{F(S_{i+1}\setminus S_{i}|S_{i})}{|S_{i+1} \setminus S_{i}|}\right| &\geq \Delta_{\min} \left(\frac{2}{n-1} - \frac{2}{n}\right) \geq \frac{2\Delta_{\min}}{n^2}.
\end{align}
This implies $O(\log (\ell_0 n^2/\Delta_{\min}))$ iterations.
\end{proof}
  
Note that in the case of flows, the algorithm is not exactly equivalent to the flow algorithm of \cite{tarjan2006balancing}, which updates flows directly. 

\section{BCD and proximal Dykstra}\label{sec:bcdykstra}
We consider the best approximation problem
\begin{equation*}
  \begin{split}
    \min\quad &\half\enorm{x-y}^2\\
    \text{s.t.}\quad & x \in C_1 \cap C_2 \cap \cdots \cap C_m.
  \end{split}
\end{equation*}
Let us show the details for only the two block case. The general case follows similarly. 

Consider the more general problem
\begin{equation}
  \label{eq.9}
  \min\quad\half\enorm{x-y}^2 + f(x)+h(x).
\end{equation}
Clearly, this problem contains the two-block best approximation problem as a special case (by setting $f$ and $h$ to be suitable indicator functions). Now introduce two variables $z,w$ that equal $x$; then the corresponding Lagrangian is
\begin{equation*}
  L(x,z,w,\nu,\mu) := \half\enorm{x-y}^2 + f(z)+h(w) + \nu^T(x-z)+\mu^T(x-w).
\end{equation*}
From this Lagrangian, a brief calculation yields the dual optimization problem
\begin{equation*}
  \min\ g(\nu,\mu) := \half\enorm{\nu+\mu-y}^2 +f^*(\nu)+h^*(\mu).
\end{equation*}
We solve this dual problem via BCD, which has the updates
\begin{equation*}
  \nu_{k+1} = \argmin\nolimits_{\nu}\ g(\nu, \mu_k),\qquad
  \mu_{k+1} = \argmin\nolimits_{\mu}\ g(\nu_{k+1},\mu).
\end{equation*}
Thus, $0 \in \nu_{k+1}+\mu_k-y + \partial f^*(\nu_{k+1})$ and $0 \in \nu_{k+1}+\mu_{k+1}-y + \partial h^*(\mu_{k+1})$. The first optimality condition may be rewritten as
\begin{equation*}
  y-\mu_k \in \nu_{k+1}+\partial f^*(\nu_{k+1}) \implies \nu_{k+1} = \prox_{f^*}(y-\mu_k) \implies \nu_{k+1} = y -\mu_k - \prox_f(y-\mu_k).
\end{equation*}
Similarly, we second condition yields $\mu_{k+1} = y - \nu_{k+1} - \prox_h(y-\nu_{k+1})$. Now use Lagrangian stationarity
\begin{equation*}
  x = y - \nu - \mu \implies y-\mu = x+\nu
\end{equation*}
to rewrite BCD using primal and dual variables to obtain the so-called proximal-Dykstra method:
\begin{equation*}
  \begin{split}
    t_k &\gets \prox_f(x_k+\nu_k)\\
    \nu_{k+1} &\gets x_k+\nu_k - t_k\\
    x_{k+1}   &\gets \prox_h(\mu_k+t_k)\\
    \mu_{k+1} &\gets \mu_k+t_k-x_{k+1}
  \end{split}
\end{equation*}

We discussed the more general problem~\eqref{eq.9} because it contains the smoothed primal as a special case, namely with $y=0$ in~\eqref{eq.9}, $f=f_1$, and $h=f_2$, we obtain
\begin{equation*}
  \min\quad f_1(x) + f_2(x) + \half\enorm{x}^2,
\end{equation*}
for which BCD yields the proximal-Dykstra method that was previously used in~\citep{barbero2011fast} for two-dimensional TV optimization.

\section{Recipe: Submodular minimization via reflections}\label{sec:recipe}
To be precise, we summarize here how to solve Problem~\eqref{eq:fullprimal} via reflections.
As we showed above, the dual is of the form
\begin{align}
  \min_{\lambda,y}\quad\enorm{y-\lambda}^2\quad\text{s.t. }\,& \lambda \in \mathcal{A} = \bigl\{(\lambda_1,\dots,\lambda_r) \in \Hc^r \mid \nlsum_{j=1}^r \lambda_j = 0\bigr\},
  \quad y \in \mathcal{B} \triangleq \prod\nolimits_{j=1}^r B(F_j). 
\end{align}
The vector $y$ consists of $r$ parts $y_j \in B(F_j)$. We first solve the dual by starting with any $z^{(0)} \in \Hc^{r}$, and iterate
\begin{align}
  z^{(k+1)} = \tfrac{1}{2}(z^{k} + R_{\mathcal{A}}R_{\mathcal{B}}(z^{(k)})).
\end{align}
Upon convergence to a point $z^*$, we extract the components
\begin{align}
  y_j = \Pi_{B(F_j)}(z_j^*).
\end{align}
The final primal solution is $x = -\sum_j y_j \in \rb^n$.

\end{document}